\documentclass[11pt]{article}
\usepackage{xcpreprint}
\ifdefined\pdfinfoomitdate\pdfinfoomitdate=1\fi
\ifdefined\pdftrailerid\pdftrailerid{}\fi
\usepackage{amsmath,amssymb}
\usepackage{booktabs}
\usepackage{graphicx}
\usepackage[round,authoryear]{natbib}
\usepackage{xurl}
\usepackage{hyperref}
\graphicspath{{figures/}}

\newcommand{\cLateGpt}{$-$0.127}
\newcommand{\cLateSeGpt}{0.056}

\newcommand{\cMedianLateGpt}{$-$0.158}
\newcommand{\cMedianLateSeGpt}{0.055}

\newcommand{\Lyx}{0.002}
\newcommand{\LyxSe}{0.042}
\newcommand{\Lyn}{0.079}
\newcommand{\LynSe}{0.022}

\newcommand{\oewsOfficialMembers}{867}
\newcommand{\oewsMissingWeights}{49}

\newcommand{\oewsMedianWeight}{42,865}
\newcommand{\cPrimaryRowsGpt}{2,466,302}
\newcommand{\cPrimaryOccGpt}{603}

\newcommand{\cMedianRowsGpt}{2,513,006}
\newcommand{\cMedianOccGpt}{617}

\newcommand{\qELate}{$-$0.156}
\newcommand{\qELateSe}{0.043}

\newcommand{\bandA}{$-$0.129}
\newcommand{\bandASe}{0.056}
\newcommand{\bandB}{$-$0.084}
\newcommand{\bandBSe}{0.039}
\newcommand{\bandC}{$-$0.002}
\newcommand{\bandCSe}{0.036}

\newcommand{\bTwoTdOne}{$-$0.168}
\newcommand{\bTwoTdSeOne}{0.047}

\newcommand{\bTwoShortHirOne}{0.171}
\newcommand{\bTwoShortHirSeOne}{0.061}

\newcommand{\bTwoTdEight}{$-$0.193}
\newcommand{\bTwoTdSeEight}{0.037}

\newcommand{\bTwoShortHirEight}{0.149}
\newcommand{\bTwoShortHirSeEight}{0.046}

\newcommand{\bTwoYoungExpNow}{$-$0.143}
\newcommand{\bTwoYoungExpSeNow}{0.050}
\newcommand{\bTwoYoungOtherNow}{0.082}
\newcommand{\bTwoYoungOtherSeNow}{0.025}
\newcommand{\bTwoOldExpNow}{0.023}
\newcommand{\bTwoOldExpSeNow}{0.017}

\newcommand{\bTwoTdNow}{$-$0.178}
\newcommand{\bTwoTdSeNow}{0.046}
\newcommand{\bTwoAllHirNow}{$-$0.046}
\newcommand{\bTwoAllHirSeNow}{0.082}
\newcommand{\bTwoAllSepNow}{$-$0.191}
\newcommand{\bTwoAllSepSeNow}{0.086}
\newcommand{\bTwoNetAllNow}{$-$0.237}
\newcommand{\bTwoNetAllSeNow}{0.049}
\newcommand{\bTwoLastHirNow}{$-$0.148}
\newcommand{\bTwoLastHirSeNow}{0.041}
\newcommand{\bTwoContSepNow}{$-$0.089}
\newcommand{\bTwoContSepSeNow}{0.033}
\newcommand{\bTwoNetDurNow}{$-$0.237}
\newcommand{\bTwoNetDurSeNow}{0.049}
\newcommand{\bTwoShortHirNow}{0.102}
\newcommand{\bTwoShortHirSeNow}{0.058}
\newcommand{\bTwoShortSepNow}{$-$0.102}
\newcommand{\bTwoShortSepSeNow}{0.058}

\newcommand{\bTwoNatTdNow}{$-$0.168}
\newcommand{\bTwoNatTdSeNow}{0.045}

\newcommand{\bTwoLogHireNow}{0.063}
\newcommand{\bTwoLogHireSeNow}{0.027}

\newcommand{\bTwoPropHireNow}{0.189}
\newcommand{\bTwoPropHireSeNow}{0.075}

\newcommand{\bTwoRiUnrestrictedEmp}{0.002}
\newcommand{\bTwoRiSectorEmp}{0.121}

\newcommand{\bTwoRiUnrestrictedLastHir}{0.002}
\newcommand{\bTwoRiSectorLastHir}{0.092}

\newcommand{\bTwoEarnTd}{$-$0.027}
\newcommand{\bTwoEarnTdSe}{0.017}

\newcommand{\rtwDiff}{$-$0.002}
\newcommand{\rtwDiffSe}{0.033}
\newcommand{\rtwDiffCiLo}{$-$0.067}
\newcommand{\rtwDiffCiHi}{0.062}
\newcommand{\crQwiNat}{$-$0.162}
\newcommand{\crQwiNatSe}{0.009}
\newcommand{\crQwiNatP}{0.001}

\newcommand{\crQwiCommon}{$-$0.202}
\newcommand{\crQwiCommonSe}{0.014}
\newcommand{\crQwiCommonP}{0.001}

\newcommand{\crCps}{$-$0.020}
\newcommand{\crCpsSe}{0.020}
\newcommand{\crCpsP}{0.332}
\newcommand{\crCpsWindow}{$-$0.127}

\newcommand{\unionDiff}{$-$0.064}
\newcommand{\unionDiffSe}{0.016}
\newcommand{\unionDiffP}{0.001}

\newcommand{\tuckerPublishedPct}{12.4}

\usepackage{array}
\newcolumntype{P}[1]{>{\raggedright\arraybackslash}p{#1}}

\begin{document}
\onehalfspacing

\xctitleblock{Young Workers in Jobs Exposed to Generative Artificial Intelligence: Employment Flows and Historical Comparisons}{Shuchen Meng\textsuperscript{1} \quad Xupeng Chen\textsuperscript{2,*}}{\textsuperscript{1}Department of Financial Engineering, New York University, New York, USA\\\textsuperscript{2}Department of Electrical and Computer Engineering, New York University, New York, USA\\\textsuperscript{*}Corresponding author. Email: \href{mailto:xc1490@nyu.edu}{xc1490@nyu.edu}}{This version: November 2026. It replaces an earlier version of this arXiv entry; see the Comments field.}
\begin{abstract}
\textbf{Objectives.} We examine the employment flows behind young workers' relative decline in jobs exposed to generative artificial intelligence and compare the pattern with earlier downturns. The industrial-relations question is whether access to these jobs is narrowing and what aggregate data can reveal about employer adjustment.
\textbf{Data and methods.} We use the Current Population Survey and Quarterly Workforce Indicators to compare employment changes across age and exposure groups. We reconcile two definitions of hiring and separation, examine historical and non-downturn periods, and assess sensitivity to aggregation and to the reference distribution used for inference.
\textbf{Results.} The recent common-state QWI log-employment triple difference is \bTwoTdNow{} and the late CPS comparison is \cLateGpt{}. The QWI result remains negative in a cell-panel specification, while the CPS estimate becomes much smaller and imprecise. Alternative flow definitions assign different hiring and separation contributions to the same net change. Earlier recession-and-aftermath periods show similar relative employment declines and larger within-quarter starts and exits. Institutional, earnings and use-type comparisons do not identify an employer mechanism, and the monetary proxy cannot exclude that channel.
\textbf{Conclusions.} The evidence is consistent with weaker access to exposed jobs among the youngest workers. It cannot distinguish first entry from changes affecting young incumbents, establish shorter job duration, or attribute the decline to generative artificial intelligence. Employer-level measures of recruitment, replacement and separation decisions are needed to distinguish these adjustments.
\end{abstract}
\xckeywords{}{generative artificial intelligence; early-career employment; internal labour markets; hiring; separations; employer training; youth scarring}
\newcommand{\figurelegend}[2]{\caption{#2}%
  \expandafter\gdef\csname legend#1\endcsname{#2}}

{\small\noindent\textit{Acknowledgements.} Claude Code and Codex assisted with drafting; full details are disclosed in the manuscript's Methods section.\par}

\section{Introduction}

A widely reported labour-market change associated with generative artificial intelligence in the United States concerns the beginning of working life. Payroll records show weaker employment among workers aged 22--25 in highly exposed occupations while employment of older workers in the same occupations continued to grow \citep{brynjolfsson2026canaries}. Administrative records show fewer young workers in exposed industries \citep{tucker2026hired}, and firm-level adoption evidence shows slower junior hiring inside adopting firms \citep{hosseini2025seniority}. These findings raise a question about employment relations. Are employers recruiting fewer young workers, leaving departures unfilled, or ending young workers' jobs more quickly?

Entry jobs provide access to training and promotion as well as current earnings. Internal labour markets organise this access through ports of entry and job ladders \citep{doeringer1971internal,osterman1987choice}. Restricting recruitment can protect incumbents while reducing opportunities for outsiders. Leaving vacancies unfilled can reduce a cohort gradually, whereas dismissals or non-renewals place losses on workers whose employment has already begun. These responses have different consequences for risk and for employer training, even when they produce the same employment stock.

We examine employment stocks and flows using two public sources. The Quarterly Workforce Indicators (QWI) report employment, hires, separations and earnings for new hires in industry--state cells through 2025. The Current Population Survey (CPS) provides occupation-level employment comparisons through September 2026. The main QWI comparison is between ages 22--24 and 35--44; the main CPS comparison is between ages 22--25 and 30--49. In each source, we compare the change in young workers' employment in highly exposed jobs with that in other jobs, then subtract the corresponding comparison for older workers. These are relative employment comparisons and need not imply a fall in the absolute number of young workers employed.

The comparisons from the two public sources show a decline concentrated among the youngest observed groups. In the common-state QWI sample, the log-employment triple difference is \bTwoTdNow{}. The exposed employment point estimate is negative for young workers, while the older exposed point estimate is positive but imprecise. The national CPS contrast is \cLateGpt{}, although young exposed employment is nearly unchanged in absolute terms. The result depends on how employment is aggregated. The cell-panel specification preserves the negative QWI contrast but reduces the CPS estimate to \crCps{}; its wild-bootstrap value is \crCpsP{}. The CPS therefore supplies an independent public-data comparison whose magnitude depends on the estimand.

The flow analysis explains why hiring and separation accounts can appear to disagree. All hires include workers who start and leave within a quarter. Hires observed on the quarter-end roster exclude those starts, and the corresponding exits also enter all separations. Within-quarter starts and exits therefore cancel from net employment. In the recent QWI comparison, hires retained at quarter end make a negative contribution to the relative employment change. The reported log-rate contrasts and the older-proportional hiring comparison are positive, so the contribution is not a proportional hiring-rate effect. Quarter-end presence records the roster at the boundary; it does not hold tenure fixed or establish shorter job duration.

Historical comparisons show how specific the recent pattern is. We apply the same exposure classification to the 2001 and 2008--2009 recession-and-aftermath episodes and restrict each episode to states observed in the earliest sample. Earlier contractions nevertheless reproduce the negative age-by-exposure interaction and larger within-quarter start-and-exit contributions. Their employment components differ from the recent point estimates because exposed employment falls for both age groups in the earlier episodes. The recent employment contrast lies outside the reported central range of non-downturn windows, including a matched-offset sensitivity. It is not unusual under reassignment of exposure within broad sectors. These comparisons do not establish statistical equality or a shared employer adjustment rule.

Our contribution is to connect the interpretation of recent employment patterns to their flow definitions and historical context. \citet{tucker2026hired} already studies QWI hiring and young-versus-older employment differences. We reconcile retained-hire and continuing-job measures with all hires and separations, then compare the resulting pattern with historical and non-downturn periods. \citet{brynjolfsson2026canaries} document an age gradient and an automation-versus-complementarity pattern in payroll data. Our CPS comparison uses a different source, window and estimand; its use-type contrast has the same direction but is imprecise and sensitive to pre-trends. Studies that link outcomes to firm or workplace adoption, including \citet{hosseini2025seniority} and \citet{humlum2026still}, observe margins that the aggregate sources cannot.

The findings are consistent with weaker opportunities for young workers to enter exposed jobs, but they do not isolate first labour-market entry from changes affecting young incumbents. Historical recurrence shows that a similar aggregate employment and gross-flow signature can accompany other shocks; it does not identify a shared employer rule or rule out an artificial-intelligence contribution. Descriptive institutional comparisons and imprecise new-hire earnings provide limited evidence about workplace adjustment, while the monetary-sensitivity exercise cannot exclude that channel. For industrial relations, the issue is how employment change is distributed between access for outsiders and security for incumbents. Resolving it requires recruitment, replacement, contract and separation information linked to adoption or use measured before the decision.

The article first develops the employment-relations predictions and relates them to the emerging evidence. It then describes the data and final comparisons, presents the employment and flow results, and examines historical, institutional and measurement limits. The discussion identifies the workplace information needed to distinguish employer responses left unresolved by aggregate data.

\begin{figure}[!htbp]
\centering
\includegraphics[width=\textwidth]{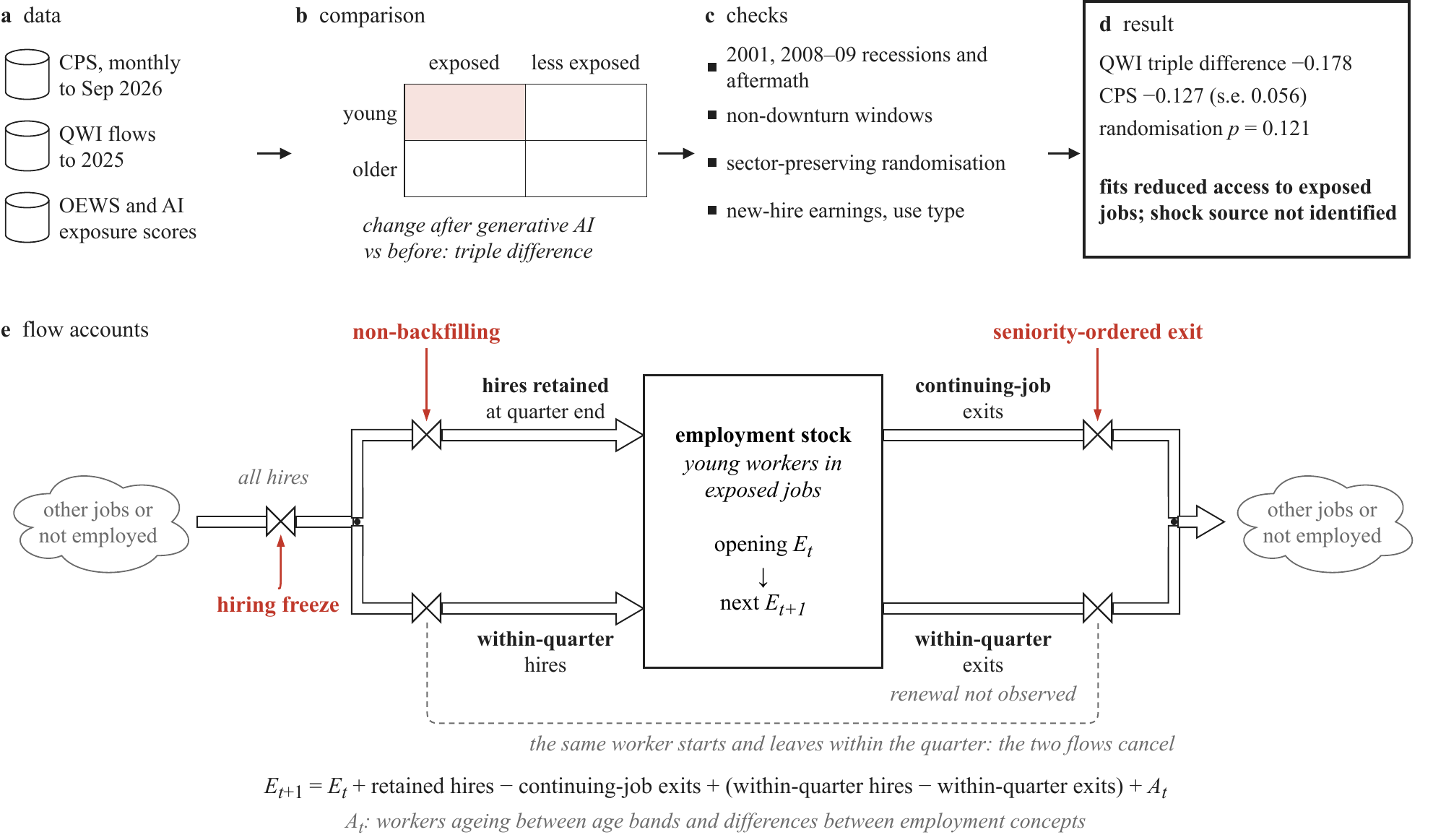}
\figurelegend{stockflow}{\textbf{Employment stocks and flows leave several employer responses observationally similar.} OEWS denotes the Occupational Employment and Wage Statistics used for staffing weights. The schematic distinguishes beginning-of-quarter employment, hires observed at quarter end, all hires, continuing-job separations and all separations. A worker who starts and leaves within the quarter appears in all hires and all separations but not in the quarter-end roster. The remainder in equation~(\ref{eq:flow}) collects ageing between age bands and differences between QWI employment concepts. Hiring freezes, non-backfilling and seniority-ordered exits are possible responses, not estimates; contract renewal is not observed.}
\label{fig:stockflow}
\end{figure}

\section{Employment relations and competing predictions}

\subsection{Ports of entry, training and early careers}

Internal labour markets use entry positions to screen workers, teach firm-specific routines and build promotion pools \citep{doeringer1971internal,osterman1987choice}. If technology performs work assigned to entrants, employers may offer fewer entry positions even when experienced workers remain employed. Task substitution and complementarity can instead preserve entry demand when systems raise the productivity of less experienced workers \citep{autor2003skill,acemoglu2022tasks,brynjolfsson2025generative,noy2023experimental}. Automating relatively inexpert tasks can also raise the expertise required for the work that remains \citep{autor2025expertise}. We observe neither first jobs, training nor later careers, so the age gradient localises the change without identifying labour-market entry.

Employer responses have different implications for workers. A hiring freeze reduces young hires without necessarily increasing separations. Non-backfilling lowers hiring after a departure and can erode an entry cohort gradually. Layoffs, non-renewals or quits increase separation rates, making the signed separation contribution more negative, and place losses on workers already in the labour market. A seniority rule can protect accumulated firm-specific capital while concentrating exits among recent hires \citep{abraham1984length}. French permanent-contract-flow evidence describes lower hiring and lower contract termination for young workers in highly exposed jobs \citep{palladino2026exposition}. Swedish employer-level evidence treats statutory seniority incentives as an alternative adjustment mechanism and does not interpret an age gradient alone as proof of technology-driven substitution \citep{lodefalk2026same}.

\subsection{Relation to the emerging evidence}

The closest payroll evidence is the August revision of \citet{brynjolfsson2026canaries}. Its age gap and automation-versus-complementarity comparison motivate our public-data comparison, but the sources, windows and estimands differ. \citet{tucker2026hired} is the closest administrative-data study. We use the same QWI setting while reconciling retained hires with all hires and separations and adding historical and non-downturn benchmarks. \citet{hosseini2025seniority} links firm adoption to slower junior hiring, whereas \citet{humlum2026still} reports no adoption-specific early-career decline in Denmark. These studies observe adoption or workplace outcomes; our data observe exposure and aggregate employment flows.

Other firm-level evidence points in both directions. Greater mean task exposure can reduce within-firm occupation demand, while concentrating exposure in fewer tasks offsets part of those losses; Copilot adoption is associated with more entry-level software-engineer hiring \citep{hampole2025artificial,baird2026copilot}. These results rule out a universal substitution account and motivate moderators that public flows cannot identify. Earlier CPS studies also locate weakness in entry from non-employment \citep{atkinson2026dallas,massenkoff2026labor}. One recent comparison finds weaker employment among young adults both with and without college degrees \citep{ozimek2026eig}. The evidence supports a measurement question and leaves the employer mechanism unresolved.

\section{Data and methods}

\subsection*{Final design and data sources}

We use basic monthly CPS data from January 2018 to September 2026 and QWI release R2026Q3 for private employment in four-digit industries by state. October 2025 was not collected in the CPS. The later CPS window is not a publication-date boundary. The QWI age comparison is 22--24 versus 35--44. The CPS age comparison is 22--25 versus 30--49; the age-gradient table uses ages 35--49 as a separate comparison band.

Occupation exposure in the CPS uses the 2018 Census classification and the May 2022 Occupational Employment and Wage Statistics (OEWS) staffing weights. Each of the \oewsOfficialMembers{} official detailed 2018 Standard Occupational Classification (SOC) occupations receives at most one Census owner. Separately listed occupations are excluded from the residual categories. For pre-2020 records, the measure uses an employment-weighted mean of its 2018 crosswalk targets with an available exposure score. The target weights use CPS employment from January 2020 through December 2022, with a unit fallback weight when no target employment weight is available. Under the non-imputing primary rule, a scored member without an observed employment weight is dropped from its family mean; a family without a scored weighted member remains unavailable. This rule does not guarantee a direction for the estimate. The former global-median imputation of \oewsMedianWeight{} is a labelled sensitivity.

QWI industry exposure is the staffing-weighted mean of exact detailed-SOC scores, with the top fifth fixed using 2022 employment. Unmatched detailed codes remain unscored. We do not impute a child-family mean. The exact detailed 2018 SOC codes from the fixed detailed occupation table define the linkage.

The Eloundou et al. GPT-4 measure is a theoretical task-exposure index: tasks whose completion time could be at least halved using the model without additional software receive full weight, while tasks requiring additional tools receive one-half weight in the capability calculation. It is not an observed-use measure. The Anthropic measure describes observed occupational use in Claude conversations. It counts theoretically feasible tasks under the Eloundou classification only when they meet a threshold of work-related Claude usage; fully automated implementations receive full weight and augmentative use receives one-half weight when task time is aggregated over all occupational task time. These half weights therefore have different meanings and are not interchangeable \citep{eloundou2024gpts,massenkoff2026labor}. Because the observed-use measure uses the same theoretical task classification as a filter, it is not an independent exposure family. The separate type-of-use analysis assigns the first 2018 target in the official Census change-block order before January 2020 and applies the documented O*NET bridge; it is not the rule for the main exposure measure.

\subsection*{Relative employment and uncertainty}

For observation period $t$ (a month in the CPS and a quarter in the QWI), define
\begin{equation}
R_t=\left[\log E^Y_{X,t}-\log E^Y_{N,t}\right]-\left[\log E^O_{X,t}-\log E^O_{N,t}\right],
\label{eq:contrast}
\end{equation}
where $Y$ and $O$ denote young and older groups, and $X$ and $N$ denote exposed and other jobs. A negative change in $R_t$ means that the exposed-to-other employment ratio fell more, or rose less, among young than among older workers. It does not determine either direct young-worker comparison and need not mean that the absolute number of exposed jobs fell. All four employment quantities must be positive.

For the CPS, the primary estimate is the change in the national window mean of monthly $R_t$. For the QWI headline, it is the change from the mean of the four base-period quarters to the single 2025Q4 endpoint, using the states observed throughout the episode. We also report a secondary national QWI window mean using all available states; it is a different estimand. A cell-panel diagnostic instead gives each observed positive occupation--state--month or industry--state--quarter cell equal weight after fixed-effect projection. These estimands answer different questions. The national log ratio cancels any proportional movement that is common to the two national age groups. It does not eliminate occupation- or industry-specific shocks shared by both ages when the age groups have different compositions within the exposure groups.

For CPS occupation comparisons, standard errors use 499 mean-one Poisson bootstrap replications by final harmonised occupation group. QWI flow estimates instead resample four-digit industries, and the cell-panel regressions report state-clustered uncertainty with wild-cluster draws. Confidence intervals and hypothesis tests are pointwise and are not adjusted for multiplicity; they do not provide simultaneous coverage across measures, windows or specifications. The former household bootstrap is retained as a labelled sensitivity, alongside state-cluster uncertainty and population-composition checks. The main exposure and type-of-use analyses use the same final harmonised occupation groups; a separate type-of-use analysis uses its first-target rule.

\subsection*{QWI flows and reference comparisons}

Beginning-of-quarter employment is $Emp$. $HirA$ and $Sep$ count every hire and separation observed during a quarter. $HirAEnd$ counts hires observed on the quarter-end roster, while $SepBeg$ counts separations from jobs already present at quarter start. The difference between all and quarter-end-retained hires is within-quarter starts not present at quarter end. The corresponding exits appear on the separation side and cancel from net employment. Quarter-end retention is not a fixed-tenure survival measure.

For group $g$, employment evolves as
\begin{equation}
E_{g,t+1}=E_{g,t}+H_{g,t}-S_{g,t}+A_{g,t},
\label{eq:flow}
\end{equation}
where $A$ includes ageing between bands and differences among QWI employment concepts. A rate divides an event count by beginning-of-quarter employment. A flow contribution is an accounting contribution to the stock change, not an employer-level causal effect. We remove quarter-of-year means using each episode's reference period and apply equation~\eqref{eq:contrast}.

Historical comparisons use the fixed industry sort in the 2001 and 2008--2009 recession-and-aftermath episodes and in the recent episode. Table~\ref{tab:episode_definitions} lists the seasonal references, bases, accumulated-flow windows and stock endpoints. Rolling non-downturn windows provide a separate reference distribution. These diagnostics do not isolate bin width. Randomisation inference permutes exposure labels either across all industries or within two-digit sectors on the common-state support. A small value is unusual under the chosen assignment rule; it does not say that artificial intelligence assigned the exposure. Right-to-work status, pre-outcome union membership and earnings for new hires who reach full-quarter stable employment are descriptive checks. The flow, placebo and institutional comparisons were specified after the main estimates were known and are exploratory.

\subsection*{Use of artificial-intelligence tools}

Claude Code (Anthropic; Claude Opus 5.5 and Claude Sonnet 5, 2026) and Codex (OpenAI; GPT-5.6-sol, October 2026) drafted text, wrote analysis and figure code, and searched the literature. They did not create or alter data. The authors take responsibility for the analysis and the sources cited.

\section{Results}

\subsection{Recent employment and the historical benchmark}

The recent QWI comparison shows a negative relative employment change concentrated among young workers. In the common-state sample, young exposed employment changes by \bTwoYoungExpNow{} (\bTwoYoungExpSeNow{}), whereas the recent older-exposed point estimate is positive at \bTwoOldExpNow{} (\bTwoOldExpSeNow{}); its normal 95 percent interval includes zero. Young employment in other industries has a positive point estimate, \bTwoYoungOtherNow{} (\bTwoYoungOtherSeNow{}). The common-state triple difference is \bTwoTdNow{} (\bTwoTdSeNow{}), close to the historical point estimates \bTwoTdOne{} (\bTwoTdSeOne{}) and \bTwoTdEight{} (\bTwoTdSeEight{}). We do not test equality across episodes.

\begin{figure}[!htbp]
\centering
\includegraphics[width=\textwidth]{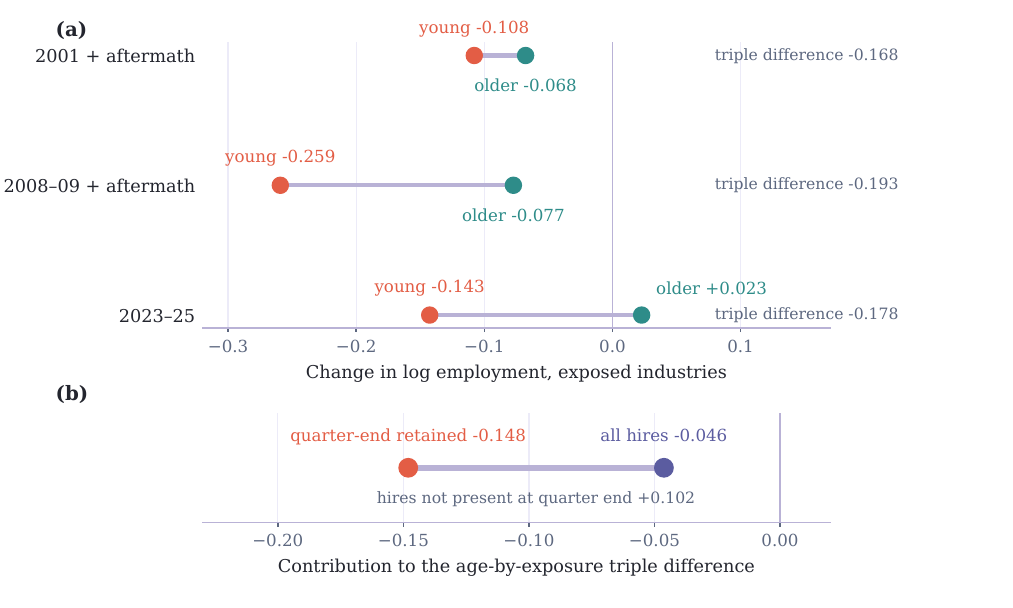}
\figurelegend{headline}{\textbf{Similar triple differences accompany different flow components.} Panel (a) shows base-to-end changes for young and older workers in exposed industries; the full age-by-exposure comparison also uses other industries. The 2001 and 2008--2009 labels refer to recession-and-aftermath episodes. Panel (b) compares all-hire and quarter-end-retained-hire contributions to the young minus older, exposed minus other triple difference. The distance between them is the within-quarter-hire contribution, with an offsetting exit-side term. Quarter-end retention records presence at the quarter boundary and is not a fixed-tenure survival measure. Under common-state sector-preserving randomisation the employment interaction has \(p=\bTwoRiSectorEmp{}\). Definitions appear in Table~\ref{tab:flows}.$^{\dagger}$}
\label{fig:headline}
\end{figure}

Earlier recession-and-aftermath episodes show the same negative age-by-exposure interaction and larger within-quarter starts and exits. Exposed employment contracts for both ages in those episodes. The recent four-component pattern differs because the older exposed point estimate is positive, although imprecise. The recent employment interaction lies below the reported central range of eligible non-downturn windows, while both flow-gap measures remain inside their broad ranges. A matched-offset sensitivity keeps these placements. The combined component assessment remains inconclusive.

\begin{table}[!htbp]\centering
\caption{Episode windows used in the historical and recent comparisons}
\label{tab:episode_definitions}\small
\resizebox{\textwidth}{!}{\begin{tabular}{@{}llllllr@{}}\toprule
Episode & Scope & Seasonal reference & Base & Accumulated flows & Stock endpoint & Flow quarters \\ \midrule
2001 episode & recession and aftermath & 1997--1999 & 2000Q1--2000Q4 & 2001Q1--2003Q4 & 2004Q1 & 12 \\
2008--09 episode & recession and aftermath & 2004--2006 & 2007Q1--2007Q4 & 2008Q1--2010Q4 & 2011Q1 & 12 \\
Recent episode & recent comparison & 2017--2019 & 2021Q4--2022Q3 & 2023Q1--2025Q3 & 2025Q4 & 11 \\
\bottomrule\end{tabular}}
\begin{minipage}{0.99\textwidth}\footnotesize\smallskip
\textit{Notes:} Every stock comparison is the change from the mean of the four-quarter base to the listed stock endpoint. Flow contributions sum the listed accumulated-flow quarters after removing quarter-of-year means from the listed three-year seasonal reference. The two historical specifications cover a recession and its aftermath; the recent specification has a one-quarter gap between its base and accumulated-flow window.
\end{minipage}\end{table}

\begin{table}[!htbp]\centering
\caption{Employment components and flow identities across episodes, common 21-state sample$^{\dagger}$}
\label{tab:flows}\small
\resizebox{\textwidth}{!}{\begin{tabular}{@{}lccc@{}}\toprule
 & 2001 episode & 2008--09 episode & Recent episode \\ \midrule
\multicolumn{4}{@{}l}{\textit{A. Four components of the employment contrast, states covered in 2001}} \\
\quad Young, exposed & $-0.108$ $(0.052)$ & $-0.259$ $(0.045)$ & $-0.143$ $(0.050)$ \\
\quad Young, other & $0.062$ $(0.032)$ & $-0.090$ $(0.032)$ & $0.082$ $(0.025)$ \\
\quad Older, exposed & $-0.068$ $(0.024)$ & $-0.077$ $(0.025)$ & $0.023$ $(0.017)$ \\
\quad Older, other & $-0.065$ $(0.013)$ & $-0.101$ $(0.020)$ & $0.069$ $(0.014)$ \\
\quad Young: exposed minus other & $-0.170$ $(0.061)$ & $-0.169$ $(0.056)$ & $-0.225$ $(0.057)$ \\
\quad Older: exposed minus other & $-0.003$ $(0.027)$ & $0.024$ $(0.031)$ & $-0.047$ $(0.022)$ \\
\quad Triple difference & $-0.168$ $(0.047)$ & $-0.193$ $(0.037)$ & $-0.178$ $(0.046)$ \\
\multicolumn{4}{@{}l}{\textit{B. All-flow identity}} \\
\quad All hires (+) & $0.028$ $(0.101)$ & $0.119$ $(0.058)$ & $-0.046$ $(0.082)$ \\
\quad All separations ($-$) & $-0.257$ $(0.095)$ & $-0.264$ $(0.065)$ & $-0.191$ $(0.086)$ \\
\quad Net all flows & $-0.229$ $(0.064)$ & $-0.145$ $(0.029)$ & $-0.237$ $(0.049)$ \\
\quad Stock--flow remainder & $0.062$ $(0.036)$ & $-0.048$ $(0.033)$ & $0.059$ $(0.069)$ \\
\quad Employment change = net flows + remainder & $-0.168$ $(0.047)$ & $-0.193$ $(0.037)$ & $-0.178$ $(0.046)$ \\
\multicolumn{4}{@{}l}{\textit{C. Quarter-end-retained-hire/continuing-job identity}} \\
\quad Hires retained at quarter end (+) & $-0.143$ $(0.063)$ & $-0.030$ $(0.027)$ & $-0.148$ $(0.041)$ \\
\quad Continuing-job separations ($-$) & $-0.087$ $(0.037)$ & $-0.115$ $(0.025)$ & $-0.089$ $(0.033)$ \\
\quad Net retained/continuing flows & $-0.230$ $(0.064)$ & $-0.145$ $(0.029)$ & $-0.237$ $(0.049)$ \\
\quad Public-data identity residual & $0.0003$ $(0.0002)$ & $-0.0000$ $(0.0002)$ & $0.0002$ $(0.0002)$ \\
\quad Stock--flow remainder & $0.062$ $(0.036)$ & $-0.048$ $(0.033)$ & $0.059$ $(0.069)$ \\
\quad Employment change = net flows + residual + remainder & $-0.168$ $(0.047)$ & $-0.193$ $(0.037)$ & $-0.178$ $(0.046)$ \\
\multicolumn{4}{@{}l}{\textit{D. Within-quarter starts and exits}} \\
\quad Hires not present at quarter end (+) & $0.171$ $(0.061)$ & $0.149$ $(0.046)$ & $0.102$ $(0.058)$ \\
\quad Corresponding within-quarter exits ($-$) & $-0.171$ $(0.061)$ & $-0.149$ $(0.046)$ & $-0.102$ $(0.058)$ \\
\quad Sum = public-data identity residual & $0.0003$ $(0.0002)$ & $-0.0000$ $(0.0002)$ & $0.0002$ $(0.0002)$ \\
\multicolumn{4}{@{}l}{\textit{E. Coverage and national comparison}} \\
\quad Common-state industry--state pairs & 2,882 & 2,660 & 2,607 \\
\quad National triple difference & $-0.168$ $(0.049)$ & $-0.191$ $(0.032)$ & $-0.168$ $(0.045)$ \\
\quad National industry--state pairs & 2,882 & 5,794 & 5,836 \\
\bottomrule\end{tabular}}
\begin{minipage}{0.95\textwidth}\footnotesize\smallskip
\textit{Notes:} QWI, 4-digit NAICS industry by state, private ownership; ages 22--24 and 35--44; the 21 states observed throughout the 2001 episode in panels A--D. Table~\ref{tab:episode_definitions} gives every base, accumulated-flow window, seasonal reference and stock endpoint; the historical columns are recession-and-aftermath episodes. Each component is the base-to-end change in log employment. The triple difference is (young exposed minus young other) minus (older exposed minus older other). Flow entries are contributions to that triple difference: hires enter positively and separations negatively. Panel B adds all hires, all separations and the stock--flow remainder. Panel C adds hires observed on the quarter-end roster, continuing-job separations, the explicitly reported public-data identity residual and the same remainder. Quarter-end retention is not a fixed-tenure survival measure. Panel D shows that within-quarter starts and exits cancel up to that residual; the residual is below 0.001 in absolute value. Panel E reports exposure-linked analytical pair counts; availability-only unlinked industries are excluded from estimates and bootstrap draws. Complete flow samples require every required sequence to be published and strictly positive; legitimate published zero flows are excluded by design. Industry bootstrap, 499 draws. $^{\dagger}$Exploratory; specified after the initial estimates were examined.
\end{minipage}\end{table}

\begin{table}[!htbp]\centering
\caption{Non-downturn placebos and randomisation inference$^{\dagger}$}
\label{tab:placebo_ri}\small
\resizebox{\textwidth}{!}{\begin{tabular}{@{}lcccc@{}}\toprule
 & & & &  \\ \midrule
\multicolumn{5}{@{}l}{\textit{A. Non-downturn 12-quarter placebo distribution, common states}} \\
 & 2.5th percentile & Median & 97.5th percentile & 2023--2025 \\
\quad Employment triple difference & $-$0.084 & 0.055 & 0.083 & $-$0.178 \\
\quad All-separation contribution / employment change & $-$11.915 & 1.341 & 7.977 & 1.072 \\
\quad Within-quarter hire contribution & $-$0.094 & $-$0.059 & 0.029 & 0.102 \\
\quad All hires + all separations & $-$0.012 & 0.027 & 0.195 & $-$0.237 \\
\quad All separations minus all hires & $-$0.304 & 0.167 & 0.269 & $-$0.144 \\
\quad Continuing-job separations minus quarter-end-retained hires & $-$0.245 & 0.054 & 0.125 & 0.059 \\
\multicolumn{5}{@{}l}{\textit{B. Randomisation inference, 999 assignments (two-sided p-values)}} \\
 & Actual & Unrestricted & Within 2-digit sector & \\
\quad Employment triple difference & $-$0.178 & 0.002 & 0.121 & \\
\quad All hires & $-$0.046 & 0.704 & 0.578 & \\
\quad All separations & $-$0.191 & 0.160 & 0.347 & \\
\quad Quarter-end-retained hires & $-$0.148 & 0.002 & 0.092 & \\
\quad Continuing-job separations & $-$0.089 & 0.044 & 0.503 & \\
\bottomrule\end{tabular}}
\begin{minipage}{0.95\textwidth}\footnotesize\smallskip
\textit{Notes:} Panel A uses 30 exploratory rolling 12-quarter outcome windows whose base and outcome quarters lie outside the 2007Q4--2009Q2 NBER recession, over 2004--2019, in the states covered in 2001. Its four-quarter base is immediately before the 12-quarter outcome and its seasonal reference is at offsets $-16$ to $-5$; the recent episode instead has a one-quarter gap, eleven flow quarters and a seasonal reference at offsets $-24$ to $-13$. The later matched-offset sensitivity is reported in the Supporting Information. Panel B uses the same common-state support as Figure~\ref{fig:headline} and reports finite-randomisation p-values $(b+1)/(999+1)$ under unrestricted reassignment of four-digit-industry exposure and reassignment within two-digit NAICS sectors. A zero-exceedance result is the corrected Monte Carlo $p=0.001$ (0/999; minimum attainable corrected value), not an upper bound on the exact assignment tail. $^{\dagger}$Specified after the initial estimates were examined.
\end{minipage}\end{table}

\subsection{Flow identity and hiring-rate interpretation}

All hires contribute \bTwoAllHirNow{} (\bTwoAllHirSeNow{}) and all separations contribute \bTwoAllSepNow{} (\bTwoAllSepSeNow{}) in the recent common-state sample. Their net is \bTwoNetAllNow{} (\bTwoNetAllSeNow{}). Quarter-end-retained hires contribute \bTwoLastHirNow{} (\bTwoLastHirSeNow{}) and continuing-job separations contribute \bTwoContSepNow{} (\bTwoContSepSeNow{}), for a net of \bTwoNetDurNow{} (\bTwoNetDurSeNow{}). The public-data identity residual and stock--flow remainder close both decompositions.

The within-quarter-hire contribution is \bTwoShortHirNow{} (\bTwoShortHirSeNow{}), with an offsetting exit-side term of \bTwoShortSepNow{} (\bTwoShortSepSeNow{}). Its normal interval includes zero. The corresponding point estimates are \bTwoShortHirOne{} (\bTwoShortHirSeOne{}) in 2001 and \bTwoShortHirEight{} (\bTwoShortHirSeEight{}) in 2008--2009. Hire timing within the quarter can produce this accounting term without changing job duration.

The all-hire log-rate contrast is \bTwoLogHireNow{} (\bTwoLogHireSeNow{}). The older-proportional hiring deviation is \bTwoPropHireNow{} (\bTwoPropHireSeNow{}). Both are positive, so the negative level-rate term is not a proportional hiring-rate effect. The Tucker reconstruction ends in 2025Q1 and differs from the published \tuckerPublishedPct{} percent decline in outcome and window.

The paired industry bootstrap gives the right-to-work-minus-other-state contrast \rtwDiff{} (\rtwDiffSe{}), with interval [\rtwDiffCiLo{}, \rtwDiffCiHi{}]. The interval includes zero.

\begin{table}[!htbp]\centering
\caption{Young-worker QWI flow comparison and right-to-work split$^{\dagger}$}
\label{tab:recon}\small
\resizebox{\textwidth}{!}{\begin{tabular}{@{}lcc@{}}\toprule
 & Estimate (s.e.) &  \\ \midrule
\multicolumn{3}{@{}l}{\textit{A. Tucker's statistics: 22--24 only, top fifth versus other industries, last four quarters to 2025Q1 versus the four ending at the reference}} \\
\quad Hires still employed at quarter end, log count & $-0.199$ $(0.049)$ & \\
\quad Hires still employed at quarter end, log rate & $-0.062$ $(0.021)$ & \\
\quad Separations of continuing jobs, log count & $-0.157$ $(0.053)$ & \\
\quad Separations of continuing jobs, log rate & $-0.012$ $(0.022)$ & \\
\multicolumn{3}{@{}l}{\textit{B. Change in log employment and its sources}} \\
 & Quarter-end-retained hires & Continuing-job separations \\
\quad 22--24 vs 35--44 (main): employment $-$0.168 & $-0.146$ $(0.034)$ & $-0.071$ $(0.037)$ \\
\quad Right-to-work states (2022): employment $-$0.170 & $-0.127$ $(0.037)$ & $-0.033$ $(0.054)$ \\
\quad Other states: employment $-$0.167 & $-0.158$ $(0.036)$ & $-0.103$ $(0.028)$ \\
\multicolumn{3}{@{}l}{\textit{C. Between-group employment contrast}} \\
 & Employment contrast (s.e.) & Normal 95\% interval \\
\quad Right-to-work minus other states & $-0.002$ $(0.033)$ & [$-$0.067, 0.062] \\
\bottomrule\end{tabular}}
\begin{minipage}{0.95\textwidth}\footnotesize\smallskip
\textit{Notes:} Panel A follows the flow concepts of \citet{tucker2026hired}: young workers only, exposed against other industries, reference quarters 2022Q3 (hires) and 2022Q4 (separations), with means over the last four quarters to 2025Q1 against the four quarters ending at the reference. The count definitions agree, but this article divides flow counts by beginning-of-quarter employment; Tucker's hiring rate uses average beginning- and end-of-quarter employment. Tucker's published 12.4 percent employment decline ends in 2025Q2 and differs in outcome, window and comparison. Panel B reports quarter-end-retained-hire and continuing-job-separation contributions; quarter-end retention is not a fixed-tenure survival measure. Panel C reports the right-to-work-minus-other-state difference in the employment triple difference. Its 499-draw paired industry bootstrap applies the same multinomial four-digit-industry weights to both state groups; the displayed interval is the point estimate plus or minus 1.96 bootstrap standard errors. This descriptive test was specified after the split estimates were known. Right-to-work status is as of 2022; Michigan, which repealed its law in 2024, drops out of the balanced sample. Alternative age bands, accumulation timing and seasonal adjustment appear in the Supporting Information. Finite-corrected randomisation inference appears in Table~\ref{tab:placebo_ri}. $^{\dagger}$Specified after the main estimates were known.
\end{minipage}\end{table}

\begin{table}[!htbp]\centering
\caption{Flow-rate and proportional-counterfactual comparisons across episodes$^{\dagger}$}
\label{tab:flow_rates}\small
\resizebox{\textwidth}{!}{\begin{tabular}{@{}lccc@{}}\toprule
 & 2001 episode & 2008--09 episode & Recent episode \\ \midrule
\multicolumn{4}{@{}l}{\textit{A. Common-state sample}} \\
\quad All-hire log-rate contrast & $0.053$ $(0.017)$ & $0.047$ $(0.013)$ & $0.063$ $(0.027)$ \\
\quad Quarter-end-retained-hire log-rate contrast & $0.034$ $(0.014)$ & $0.019$ $(0.011)$ & $0.069$ $(0.021)$ \\
\quad All-separation log-rate contrast & $0.054$ $(0.015)$ & $0.054$ $(0.014)$ & $0.043$ $(0.023)$ \\
\quad Continuing-job-separation log-rate contrast & $0.039$ $(0.010)$ & $0.036$ $(0.011)$ & $0.051$ $(0.014)$ \\
\quad Young hiring deviation from older-proportional benchmark & $0.180$ $(0.066)$ & $0.145$ $(0.043)$ & $0.189$ $(0.075)$ \\
\quad Young separation contribution relative to older-proportional benchmark & $-0.208$ $(0.059)$ & $-0.172$ $(0.044)$ & $-0.143$ $(0.056)$ \\
\multicolumn{4}{@{}l}{\textit{B. Available national sample}} \\
\quad All-hire log-rate contrast & $0.053$ $(0.018)$ & $0.048$ $(0.011)$ & $0.066$ $(0.023)$ \\
\quad Quarter-end-retained-hire log-rate contrast & $0.034$ $(0.015)$ & $0.023$ $(0.010)$ & $0.069$ $(0.019)$ \\
\quad All-separation log-rate contrast & $0.054$ $(0.016)$ & $0.051$ $(0.013)$ & $0.040$ $(0.023)$ \\
\quad Continuing-job-separation log-rate contrast & $0.039$ $(0.011)$ & $0.033$ $(0.011)$ & $0.044$ $(0.015)$ \\
\quad Young hiring deviation from older-proportional benchmark & $0.180$ $(0.072)$ & $0.150$ $(0.037)$ & $0.213$ $(0.064)$ \\
\quad Young separation contribution relative to older-proportional benchmark & $-0.208$ $(0.060)$ & $-0.165$ $(0.039)$ & $-0.161$ $(0.057)$ \\
\bottomrule\end{tabular}}
\begin{minipage}{0.95\textwidth}\footnotesize\smallskip
\textit{Notes:} All rows use the same analytical support, exposure sort and age groups. 2001 episode: seasonal reference 1997--1999, base 2000Q1--2000Q4, flows 2001Q1--2003Q4, stock endpoint 2004Q1; 2008--09 episode: seasonal reference 2004--2006, base 2007Q1--2007Q4, flows 2008Q1--2010Q4, stock endpoint 2011Q1; Recent episode: seasonal reference 2017--2019, base 2021Q4--2022Q3, flows 2023Q1--2025Q3, stock endpoint 2025Q4. Log-rate rows compare log quarterly flows divided by beginning-of-quarter employment. The proportional counterfactual holds each young group's base rate fixed and lets it move in proportion to the corresponding older rate; positive hiring values mean young hiring exceeded that benchmark. These quantities are not the level-rate contributions in the stock--flow identity. Industry bootstrap, 499 draws. $^{\dagger}$Exploratory; specified after the initial estimates were examined.
\end{minipage}\end{table}

Randomisation inference depends on the assignment rule. Under unrestricted reassignment, the employment and quarter-end-retained-hire values are \bTwoRiUnrestrictedEmp{} and \bTwoRiUnrestrictedLastHir{}. Within two-digit sectors they are \bTwoRiSectorEmp{} and \bTwoRiSectorLastHir{}, neither below the nominal 5 percent threshold. Changing the reference distribution does not estimate a share explained by sector composition.

\subsection{Public-data comparison and aggregation}

The CPS comparison uses an independent public source and differs from payroll evidence in its window and endpoints. The late GPT-4 exposure contrast is \cLateGpt{} (\cLateSeGpt{}), and the former median-imputation sensitivity is \cMedianLateGpt{} (\cMedianLateSeGpt{}). The observed-use and capability contrasts remain less precise. Young exposed employment changes by \Lyx{} (\LyxSe{}), while other young employment grows by \Lyn{} (\LynSe{}); exposed employment among older workers also grows. The ratio measures relative allocation and does not measure absolute job loss.

The primary CPS sample contains \cPrimaryRowsGpt{} person-month rows with \cPrimaryOccGpt{} distinct source occupation-code values across the 2010/2018 classifications. The median-imputation support contains \cMedianRowsGpt{} rows and \cMedianOccGpt{} source code values.

\begin{table}[!htbp]\centering
\caption{Relative employment of 22- to 25-year-olds in the most exposed occupations, Current Population Survey 2018--2026}
\label{tab:cps}\small
\resizebox{\textwidth}{!}{\begin{tabular}{@{}lccc@{}}\toprule
 & GPT-4 exposure & Anthropic observed exposure & Capability exposure \\ \midrule
\multicolumn{4}{@{}l}{\textit{A. Change relative to January--October 2022}} \\
\quad Nov 2022--Mar 2024$^{\dagger}$ & $-0.022$ $(0.044)$ & $-0.036$ $(0.040)$ & $-0.016$ $(0.042)$ \\
\quad Apr 2024--Jun 2025$^{\dagger}$ & $-0.014$ $(0.043)$ & $-0.027$ $(0.039)$ & $0.027$ $(0.050)$ \\
\quad Nov 2022--Jun 2025 & $-0.018$ $(0.040)$ & $-0.032$ $(0.035)$ & $0.004$ $(0.042)$ \\
\quad \textbf{Jul 2025--Sep 2026} & $-0.127$ $(0.056)$ & $-0.076$ $(0.055)$ & $-0.073$ $(0.053)$ \\
\multicolumn{4}{@{}l}{\textit{Household-bootstrap sensitivity, Jul 2025--Sep 2026$^{\dagger}$}} \\
\quad Former household clustering & $-0.127$ $(0.040)$ & $-0.076$ $(0.040)$ & $-0.073$ $(0.044)$ \\
\multicolumn{4}{@{}l}{\textit{Missing-OEWS-member-weight sensitivity, Jul 2025--Sep 2026$^{\dagger}$}} \\
\quad Former global-median imputation (sensitivity) & $-0.158$ $(0.055)$ & $-0.076$ $(0.055)$ & $-0.079$ $(0.055)$ \\
\multicolumn{4}{@{}l}{\textit{B. Jul 2025--Sep 2026, other bases, trends and months$^{\dagger}$}} \\
\quad Relative to 2019 & $-0.119$ $(0.081)$ & $-0.069$ $(0.051)$ & $-0.067$ $(0.073)$ \\
\quad Net of a 2018--2022 linear trend & $-0.133$ $(0.076)$ & $-0.092$ $(0.087)$ & $-0.103$ $(0.076)$ \\
\quad Net of a 2018--2022 trend fitted without Mar 2020--Dec 2021 & $-0.118$ $(0.080)$ & $-0.058$ $(0.092)$ & $-0.106$ $(0.078)$ \\
\quad November 2025 dropped & $-0.130$ $(0.057)$ & $-0.075$ $(0.057)$ & $-0.074$ $(0.055)$ \\
\quad State-cluster bootstrap s.e. & $(0.038)$ & $(0.028)$ & $(0.045)$ \\
\multicolumn{4}{@{}l}{\textit{C. Jul 2025--Sep 2026, samples$^{\dagger}$}} \\
\quad College graduates only & $-0.122$ $(0.083)$ & $-0.074$ $(0.077)$ & $0.011$ $(0.076)$ \\
\quad Occupations whose codes did not change in 2020 & $-0.111$ $(0.064)$ & $-0.051$ $(0.062)$ & $-0.083$ $(0.070)$ \\
\quad Most cyclically sensitive fifth of occupations excluded & $-0.092$ $(0.064)$ & $-0.084$ $(0.060)$ & $-0.066$ $(0.069)$ \\
\multicolumn{4}{@{}l}{\textit{D. Jul 2025--Sep 2026, weights$^{\dagger}$}} \\
\quad Unweighted respondent counts & $-0.112$ $(0.053)$ & $-0.067$ $(0.049)$ & $-0.061$ $(0.051)$ \\
\quad Weights held at 2024 age-by-sex composition & $-0.129$ $(0.058)$ & $-0.080$ $(0.057)$ & $-0.079$ $(0.055)$ \\
\quad Abnormal December-to-January jumps removed & $-0.147$ $(0.145)$ & $-0.042$ $(0.139)$ & $-0.040$ $(0.141)$ \\
\multicolumn{4}{@{}l}{\textit{E. Jul 2025--Sep 2026, components of the contrast (change in log employment)$^{\dagger}$}} \\
\quad 22--25, exposed occupations & $0.002$ $(0.042)$ & $0.014$ $(0.046)$ & $0.035$ $(0.049)$ \\
\quad 22--25, other occupations & $0.079$ $(0.022)$ & $0.077$ $(0.022)$ & $0.068$ $(0.022)$ \\
\quad 30--49, exposed occupations & $0.086$ $(0.037)$ & $0.053$ $(0.032)$ & $0.077$ $(0.023)$ \\
\quad 30--49, other occupations & $0.036$ $(0.011)$ & $0.040$ $(0.011)$ & $0.037$ $(0.013)$ \\
\multicolumn{4}{@{}l}{\textit{F. Unemployment share among labour-force members, Jul 2025--Sep 2026}} \\
\quad Young minus prime-age, exposed minus other & $0.001$ $(0.007)$ & $-0.001$ $(0.006)$ & $0.002$ $(0.007)$ \\
\bottomrule\end{tabular}}
\begin{minipage}{0.95\textwidth}\scriptsize\smallskip
\textit{Notes:} Monthly CPS basic data, January 2018--September 2026 (October 2025 was not collected). Estimates in panels A--D are window means of the triple difference $R_t$ (equation~\eqref{eq:contrast}), excluding rows that report standard errors alone: log employment of 22- to 25-year-olds relative to 30- to 49-year-olds, top exposure quintile of occupations relative to other occupations, minus its mean in the base period. Panel E instead reports the four component log-employment changes, and panel F reports the difference-in-differences change in the unemployment share among labour-force participants. Standard errors use 499 mean-one Poisson replications by final harmonised occupation group, defined as a connected component of the Census occupation crosswalk; all records linked across classification vintages remain in one group. The separately labelled household-bootstrap row retains the former inference as a sensitivity. Intervals are pointwise and unadjusted for multiplicity, rather than simultaneous across measures, windows or specifications. The minimum detectable effect with 80 percent power for Nov 2022--Jun 2025 (GPT-4) is 0.113. Panel D: unweighted counts; person weights rescaled each month so that totals by single year of age and sex equal their 2024 means; each January's change in $R_t$ (2019--2026) in excess of the 2016--2019 average December-to-January change removed from all later months. Panel E: the four components whose difference of differences is $R_t$. The employment-to-population ratio of 22- to 25-year-olds was 0.737 in 2019, 0.723 in January--October 2022 and 0.720 in July 2025--September 2026. The original time-stamped plan specified the Nov 2022--Jun 2025 and Jul 2025--Sep 2026 point contrasts and the unemployment-share contrast. The first two subdivisions of panel A, all bootstrap inference and panels B--E were added after the main estimates were known; daggered entries are post hoc. The correction fixes the official detailed-SOC universe at 867 members. The May-2022 OEWS data contain 830 positive weight keys: 818 match the official universe and 12 lie outside it. Of the official members, 49 lack a weight; 19 are military and 30 civilian. The non-imputing primary drops unavailable-weight members within each Census family; this rule does not guarantee a direction for the estimate. The labelled sensitivity reproduces the former global-median imputation (42,865) with the same 499-replication harmonised-occupation-group bootstrap. The Supporting Information reports the affected families, sample counts and employment coverage for every measure.
\end{minipage}\end{table}

\begin{table}[!htbp]\centering
\caption{Employment and hires of 22- to 24-year-olds by industry exposure, Quarterly Workforce Indicators 2017--2025$^{\dagger}$}
\label{tab:qwi}\small
\resizebox{\textwidth}{!}{\begin{tabular}{@{}lccc@{}}\toprule
 & GPT-4 exposure & Anthropic observed exposure & Capability exposure \\ \midrule
\multicolumn{4}{@{}l}{\textit{A. Beginning-of-quarter employment, 22--24 versus 35--44, top industry quintile versus rest}} \\
\quad 2023Q1--2025Q2 & $-0.076$ $(0.029)$ & $-0.094$ $(0.022)$ & $-0.061$ $(0.030)$ \\
\quad \textbf{2025Q3--Q4} & $-0.156$ $(0.043)$ & $-0.144$ $(0.037)$ & $-0.138$ $(0.043)$ \\
\quad 22--24 only, top versus rest, 2025Q2 & $-0.193$ $(0.050)$ & $-0.193$ $(0.044)$ & $-0.139$ $(0.057)$ \\
\multicolumn{4}{@{}l}{\textit{B. Hires still employed at quarter end, 22--24 versus 35--44, top industry quintile versus rest}} \\
\quad 2023Q1--2025Q2 & $-0.009$ $(0.020)$ & $0.004$ $(0.019)$ & $0.003$ $(0.022)$ \\
\quad 2025Q3 & $-0.041$ $(0.036)$ & $-0.020$ $(0.038)$ & $-0.020$ $(0.036)$ \\
\quad 22--24 only, top versus rest, 2025Q2 & $-0.131$ $(0.089)$ & $-0.208$ $(0.075)$ & $-0.144$ $(0.086)$ \\
\bottomrule\end{tabular}}
\begin{minipage}{0.95\textwidth}\footnotesize\smallskip
\textit{Notes:} QWI release R2026Q3, indicators Emp (beginning-of-quarter employment) and HirAEnd (hires still employed at the end of the quarter), 4-digit NAICS industry by state, private ownership, all firm ages and sizes, published values only, not seasonally adjusted; each series has its 2017--2019 quarter-of-year means removed. Employment sample: 8,623 industry--state pairs with published employment for all three age groups in every quarter (82.9 percent of 2022 employment summed across published four-digit industries); hires sample: 6,264 pairs (81.4 percent). Counts and coverage refer to exposure-linked analytical pairs; availability-complete but unlinked industries are excluded. The published all-industry aggregate code 00 is used only to compare totals and is not added to the detailed-industry denominator. Each cell is the window mean of the log triple difference minus its 2021Q4--2022Q3 mean; the last row of each panel drops the older group. Industry exposure is the May 2022 OEWS employment-weighted mean of occupation exposure over the industry's staffing; top quintile weighted by 2022 QWI employment. GPT-4 scores are joined on exact detailed 2018 SOC keys and cover 92.8 percent of staffing employment. Across all unmatched rows, 59 distinct codes remain unscored; of these, 7 broad OEWS composite codes account for 4.0 percent of staffing employment and are not assigned a child-family mean. After the main results were known, all three measures were linked to exact canonical detailed-SOC scores. Capability scores cover 92.9 percent of staffing employment; missing scores remain unavailable. The capability sort and its bootstrap results were refreshed; GPT-4 and observed-use estimates are unchanged by this correction. Standard errors from 499 bootstrap draws of analytical industries. $^{\dagger}$Specified after the main estimates were known; the detailed-SOC corrections were made after all results were known.
\end{minipage}\end{table}

The employment-complete late-window row, \qELate{} (\qELateSe{}), uses a different completeness rule and window from the common-state headline, \bTwoTdNow{} (\bTwoTdSeNow{}), and the common-flow national estimate, \bTwoNatTdNow{} (\bTwoNatTdSeNow{}). These are not alternative estimates of the same quantity.

The QWI result remains negative when the typical observed cell receives equal weight. National QWI gives \crQwiNat{} (\crQwiNatSe{}; wild \(p=\crQwiNatP{}\)) and the common-state estimate is \crQwiCommon{} (\crQwiCommonSe{}; wild \(p=\crQwiCommonP{}\)). The CPS cell-panel estimate is \crCps{} (\crCpsSe{}; wild \(p=\crCpsP{}\)), compared with the national aggregate \crCpsWindow{}. The CPS corroboration is therefore sensitive to aggregation, while the QWI direction survives the typical-cell estimand.

\begin{table}[!htbp]\centering
\caption{Cell-panel and aggregate triple-difference estimates$^{\dagger}$}
\label{tab:regression_ddd}\small
\resizebox{\textwidth}{!}{\begin{tabular}{@{}lcccc@{}}\toprule
Sample & Cell-panel estimate (state-clustered s.e.) & Wild \(p\) & Aggregate comparison & Descriptive difference \\ \midrule
QWI national headline & $-0.162$ $(0.009)$ & 0.001 & $-$0.156 & 0.0063 \\
QWI common 21-state historical sample & $-0.202$ $(0.014)$ & 0.001 & $-$0.178 & 0.0246 \\
CPS national headline & $-0.020$ $(0.020)$ & 0.332 & $-$0.127 & 0.1073 \\
\bottomrule\end{tabular}}
\begin{minipage}{0.95\textwidth}\footnotesize\smallskip
\textit{Notes:} The regression outcome is log employment in published positive cells. QWI absorbs industry-by-age and state-by-quarter effects; CPS absorbs occupation-by-age and state-by-month effects. Both include exposed-by-post and young-by-post terms. Wild-cluster \(p\)-values use 999 null-imposed Rademacher draws; every bootstrap outcome is re-projected on the original fixed effects before its clustered standard error is computed. Aggregate comparisons first sum employment within age--exposure groups. The CPS national comparison uses mean monthly log contrasts in July 2025--September 2026 against January--October 2022. Both QWI rows use the 2021Q4--2022Q3 base mean; the national row uses the mean of 2025Q3--2025Q4, while the common 21-state row uses the single 2025Q4 endpoint. QWI contrasts are seasonally adjusted using 2017--2019 quarter-of-year means. Aggregate comparisons have no state-cluster inference and differ from the cell-panel specification. The final column is the absolute difference between these distinct estimands and is not a statistical equivalence test. $^{\dagger}$Exploratory; specified after an interim diagnostic.
\end{minipage}\end{table}

\begin{figure}[!htbp]
\centering
\includegraphics[width=\textwidth]{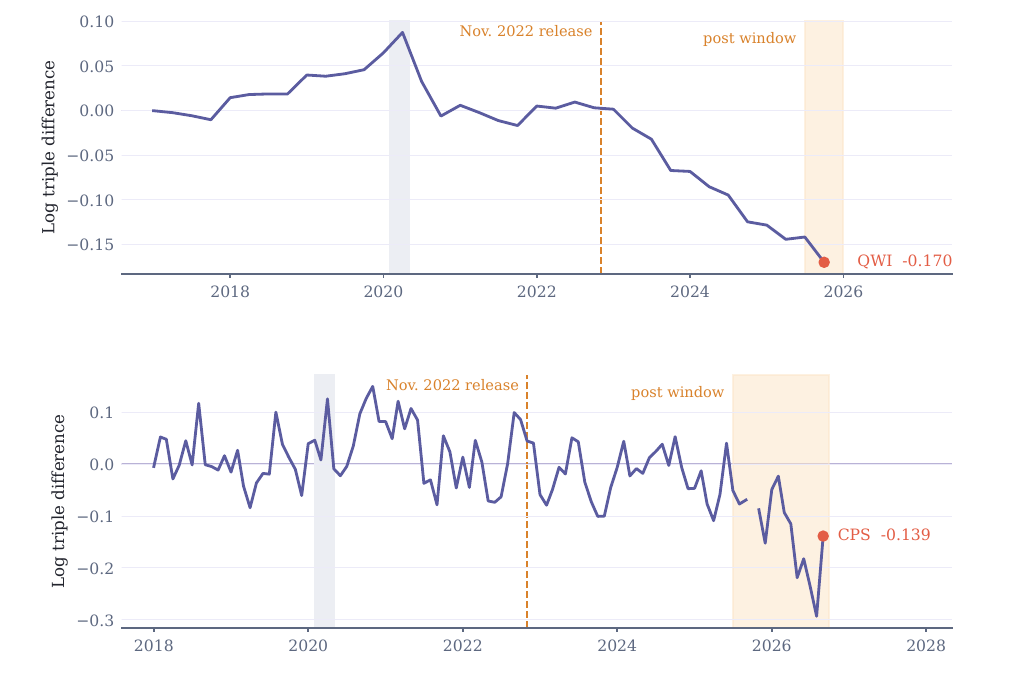}
\figurelegend{pretrends}{\textbf{Employment contrasts over time.} Each line is the young-minus-older, exposed-minus-other log-employment contrast, normalised to its source-specific pre-release base. The grey band marks the 2020 recession, the dashed line the November 2022 public release and the pale-orange band the outcome window. The CPS break is the missing October 2025 observation. Direct labels show endpoints rather than window means. The Supporting Information reports quantitative pre-trend sensitivity.$^{\dagger}$}
\label{fig:pretrends}
\end{figure}

The age gradient is concentrated among the youngest observed groups. The CPS estimates are \bandA{} (\bandASe{}) for ages 22--25, \bandB{} (\bandBSe{}) for ages 26--29 and \bandC{} (\bandCSe{}) for ages 30--34. The QWI gradient also attenuates above its youngest published band. These estimates are compatible with weaker entry opportunities and do not distinguish first entry from adjustment among young incumbents.

\begin{figure}[!htbp]
\centering
\includegraphics[width=\textwidth]{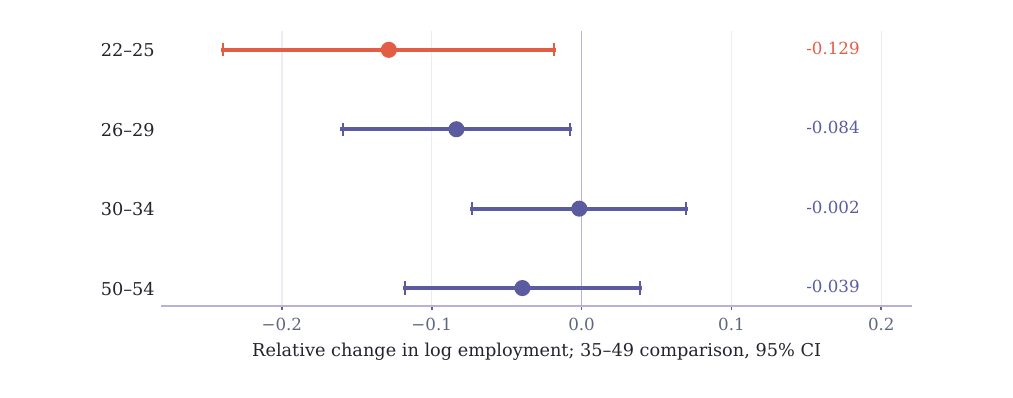}
\figurelegend{agegradient}{\textbf{The relative employment gap is smaller in older age groups.} Dots and pointwise 95 percent intervals show the late-window GPT-4 exposure contrast for each age band against ages 35--49. The CPS outcome window is July 2025--September 2026 and the base is January--October 2022. The comparisons are exploratory and were specified after the headline estimates; Table~\ref{tab:ages} reports the values.$^{\dagger}$}
\label{fig:agegradient}
\end{figure}

\begin{table}[!htbp]\centering
\caption{The relative decline is concentrated among the youngest workers: change by age band in the most exposed jobs$^{\dagger}$}
\label{tab:ages}\small
\resizebox{\textwidth}{!}{\begin{tabular}{@{}lcc@{}}\toprule
Age band & Change (s.e.) & Comparison band \\ \midrule
\multicolumn{3}{@{}l}{\textit{A. CPS, top GPT-4 quintile of occupations, Jul 2025--Sep 2026 vs Jan--Oct 2022}} \\
\quad 22--25 & $-0.129$ $(0.056)$ & 35--49 \\
\quad 26--29 & $-0.084$ $(0.039)$ & 35--49 \\
\quad 30--34 & $-0.002$ $(0.036)$ & 35--49 \\
\quad 50--54 & $-0.039$ $(0.040)$ & 35--49 \\
\multicolumn{3}{@{}l}{\textit{B. QWI, top fifth of industries, 2025Q3--Q4 vs 2021Q4--2022Q3}} \\
\quad 22--24 & $-0.154$ $(0.043)$ & 35--44 \\
\quad 25--34 & $-0.029$ $(0.018)$ & 35--44 \\
\quad 45--54 & $-0.006$ $(0.007)$ & 35--44 \\
\bottomrule\end{tabular}}
\begin{minipage}{0.92\textwidth}\footnotesize\smallskip
\textit{Notes:} Each row is a triple difference in log employment: the age band against the comparison band, exposed against other occupations (CPS) or industries (QWI). CPS: final-harmonised-occupation-group bootstrap standard errors (499 replications); the former household bootstrap is retained as a sensitivity. QWI: industry--state pairs with published values for every flow in every quarter (5,836 pairs for the 22--24 row), with every required employment and flow sequence published and strictly positive; legitimate published zero flows are excluded by design. Quarter-of-year means from 2017--2019 are removed, with 499 bootstrap draws of industries. $^{\dagger}$Specified after the main estimates were known.
\end{minipage}\end{table}

\subsection{Institutional and entry terms}

The union-density split detects descriptive heterogeneity. The above-minus-below contrast is \unionDiff{} (\unionDiffSe{}; wild \(p=\unionDiffP{}\)); it is more negative in above-median cells, but it is post-result and cannot identify bargaining as the cause. Right-to-work status is a separate descriptive split with an interval that includes zero. Earnings for new hires reaching full-quarter stable employment are \bTwoEarnTd{} (\bTwoEarnTdSe{}), which is imprecise. Matched transitions isolate no entry channel, and the automating-minus-complementary contrast has an interval that includes zero and is trend-sensitive. The interest-rate first stage is weak and its sensitivity sort is wrong-signed, so the proxy cannot exclude monetary transmission.

\begin{table}[!htbp]\centering
\caption{Young-worker employment contrast by pre-outcome union density$^{\dagger}$}
\label{tab:union_density}\small
\resizebox{\textwidth}{!}{\begin{tabular}{@{}lccc@{}}\toprule
 & Estimate (state-clustered s.e.) & Wild \(p\) & States \\ \midrule
Above-median union membership & $-0.227$ $(0.020)$ & 0.001 & 49 \\
Below-median union membership & $-0.163$ $(0.010)$ & 0.001 & 49 \\
Above minus below & $-0.064$ $(0.016)$ & 0.001 & 49 \\
\bottomrule\end{tabular}}
\begin{minipage}{0.95\textwidth}\footnotesize\smallskip
\textit{Notes:} Union membership is weighted by the CPS outgoing-rotation survey weight among employed wage-and-salary workers, pooled January 2018--October 2022 at state by two-digit NAICS. Cells require at least 100 respondents. The median (0.049) was fixed before linking QWI outcomes. Usable cells cover 97.4 percent of above-median and 96.9 percent of below-median linked base employment. Estimates use the paper's balanced QWI headline sample and are descriptive, not causal. Wild \(p\)-values use 999 Rademacher multiplier draws on state influence functions. $^{\dagger}$Exploratory; specified after an interim diagnostic. The outgoing-rotation weight was subsequently corrected.
\end{minipage}\end{table}

\begin{table}[!htbp]\centering
\caption{Earnings of new hires reaching full-quarter stable employment, 2023--2025$^{\dagger}$}
\label{tab:entry_terms}\small
\resizebox{\textwidth}{!}{\begin{tabular}{@{}lc@{}}\toprule
 &  \\ \midrule
\multicolumn{2}{@{}l}{\textit{Average monthly earnings of new hires reaching stable employment}} \\
 & Change in log earnings (bootstrap s.e.) \\
\quad Young, exposed & $0.011$ $(0.014)$ \\
\quad Young, other & $0.048$ $(0.009)$ \\
\quad Older, exposed & $0.024$ $(0.008)$ \\
\quad Older, other & $0.035$ $(0.014)$ \\
\quad Young: exposed minus other & $-0.037$ $(0.017)$ \\
\quad Older: exposed minus other & $-0.010$ $(0.016)$ \\
\quad Triple difference & $-0.027$ $(0.017)$ \\
\quad Industry--state pairs & 5,509 \\
\quad States; industries & 49; 210 \\
\bottomrule\end{tabular}}
\begin{minipage}{0.95\textwidth}\footnotesize\smallskip
\textit{Notes:} QWI EarnHirNS, the average monthly earnings of new hires who reach full-quarter stable employment, aggregated with HirNS counts as weights. Each row is the change from the 2021Q4--2022Q3 base to the 2023Q1--2025Q3 outcome window. Ages are 22--24 and 35--44; exposed industries are the fixed top fifth by May 2022 GPT-4 exposure. The triple difference is (young exposed minus young other) minus (older exposed minus older other). Industry-bootstrap standard errors, 499 draws. $^{\dagger}$Exploratory; specified after the initial estimates were examined.
\end{minipage}\end{table}

\section{Discussion}

The central result is a relative employment decline concentrated among young workers in exposed jobs. The four components show why this does not identify a general loss of young employment. Earlier recession-and-aftermath episodes reproduce a similar aggregate employment and gross-flow signature, so the historical comparison does not identify a common employer adjustment rule. The recent older-exposed point estimate is positive but imprecise, and the episodes are not formal equality tests.

The flow identity resolves the apparent hiring-versus-separations disagreement. Within-quarter starts and exits enter all gross flows and cancel from net employment. Quarter-end-retained hires and continuing-job separations answer a different accounting question from all hires and all separations. Counts, rates and proportional counterfactuals use different denominators. None of these decompositions identifies job duration, an employer's motive or first labour-market entry.

The institutional and competing-explanation checks narrow interpretation. Union density is descriptive and post-result, new-hire earnings are selected and imprecise, matched transitions do not isolate a hiring channel, and the use-type contrast is trend-sensitive. Permanent-contract-flow evidence, firm adoption studies and employer-level seniority evidence observe margins that these public aggregates cannot. The monetary proxy has a weak first stage and cannot exclude monetary transmission. The aggregate employment and gross-flow signature therefore does not identify a common employer adjustment rule.

\subsection{Limitations}

The data do not observe workplace adoption, bargaining, separation reason, contract renewal, tenure, training or promotion. Age is not a measure of first employment or job duration, and quarter-end retention is not a survival probability. CPS cells for young workers are small, occupation codes change, the population-control adjustment is noisy, and national exposure is assigned to every workplace in an industry. The CPS extract also lacks a consistent enrolment measure for the full young age band, so changes in school participation can alter its composition. QWI contains public-use noise and a stock--flow remainder. The institutional and earnings analyses require additional public inputs for exact re-estimation. These limitations leave the recent shock source and the employer adjustment rule unresolved.

The industrial-relations implication is a measurement agenda. A hiring-freeze account requires vacancies and accepted starts. Non-backfilling requires the replacement decision after an exit. Separation reasons, tenure and contract-renewal information would clarify why jobs end and how long they last. Training and promotion records would show what a smaller entry cohort means for internal labour markets. Linking these measures to adoption or use observed before employer decisions would distinguish access for outsiders from security for incumbents. Those data would test whether the aggregate employment and gross-flow signature reflects a common employer adjustment rule.

\section{Conclusion}

Young workers became less represented in jobs exposed to generative artificial intelligence in the national QWI and CPS comparisons. The change is concentrated among the youngest observed age groups and is a relative employment pattern, not a general measure of absolute job loss. The QWI contrast remains negative in a cell panel. The CPS estimate becomes much smaller and imprecise, which limits the strength of the cross-source corroboration.

Employment flows clarify the accounting without identifying the employer decisions that produced the pattern. Quarter-end-retained hires contribute negatively to the recent relative employment change, while within-quarter starts and exits offset each other. Using all hires and separations reallocates the gross contributions without changing the net flow. Quarter-end retention does not measure survival over a fixed tenure, so the data remain consistent with weaker access to exposed jobs but cannot distinguish first entry from adjustment among young incumbents.

Earlier recession-and-aftermath episodes produce similar negative age-by-exposure contrasts and larger within-quarter turnover point estimates. The employment components differ, and the recent contrast is not unusual under exposure reassignment within broad sectors. Descriptive institutional differences, imprecise new-hire earnings and the trend-sensitive use-type comparison do not identify a workplace mechanism. The monetary proxy cannot exclude monetary transmission. The evidence therefore does not identify the source of the recent change or a shared employer adjustment rule. A similar aggregate employment and gross-flow signature can arise from several employer responses.

The industrial-relations implication is that reduced employment among young workers should be examined alongside access to jobs and the treatment of incumbents. Recruitment and replacement decisions would distinguish a hiring restriction from non-backfilling. Separation reasons, tenure and contract-renewal information would clarify how jobs end, while training and promotion records would show what a smaller entry cohort means for internal labour markets. Aggregate public data cannot supply these distinctions.

\section*{Data and code availability}
All data used are public. CPS basic monthly microdata, including outgoing-rotation observations for union membership, are distributed by the U.S. Census Bureau \citep{census2026cpsbasic}; QWI release R2026Q3 is available from the Census Bureau's LEHD archive \citep{census2026qwi}; and the May 2022 national industry-by-occupation estimates are available from the Bureau of Labor Statistics \citep{bls2026oews}. The remaining inputs are the Anthropic Economic Index and observed-exposure measure of \citet{massenkoff2026labor}, and the exposure measures of \citet{eloundou2024gpts} and \citet{felten2021occupational}. The replication package includes code, derived data and stored results to regenerate the tables and figures and rerun analyses supported by the bundled derivatives. Re-estimating the union-density and new-hire-earnings results requires additional public CPS outgoing-rotation and state QWI inputs from the same cited archives; download instructions accompany the package.

\appendix
\section{Analysis plans, deviations and corrections}

The table records the analysis history without using it to strengthen the estimates. Full construction details and sensitivity results remain in the Supporting Information.

\begin{table}[!htbp]
\centering\small
\caption{Analysis plans, deviations and corrections}
\label{tab:analysis_history}
\begin{tabular}{@{}P{.2\textwidth}P{.2\textwidth}P{.3\textwidth}P{.2\textwidth}@{}}
\toprule
Analysis & Timing relative to main results & Deviation or correction & Effect on reported results \\
\midrule
CPS occupation weights & The missing-member treatment was found to be unspecified after the earlier estimates were known; the correction rule was fixed before recomputation & The primary rule drops unavailable members within a family. The former global-median imputation remains a sensitivity. & The CPS point estimate and uncertainty are reported under both rules. \\
Detailed-SOC exposure linkage & Correction after all results were known & Exact detailed 2018 SOC keys replace the historical-panel linkage; unmatched codes remain unavailable. & GPT-4 and observed-use results are unchanged; the capability estimate is recalculated. \\
QWI capability linkage & Correction after the main results & The canonical detailed-SOC source replaces the former harmonised panel. & The secondary capability estimate changes; primary GPT-4 results do not. \\
Flow and placebo comparisons & After the initial estimates were examined & The flow-identity and placebo rules were specified subsequently. & The reported comparisons remain exploratory and descriptive. \\
Union and right-to-work checks & Added after the split estimates were known & The union-density split and paired industry bootstrap provide descriptive contrasts. & Both checks remain non-causal and their intervals are reported. \\
Population-composition check & The plan specified five-year age bands crossed with sex & The implemented check uses single-year ages crossed with sex at the 2024 mean. & The check is retained as a weak sensitivity because uniform rescaling largely cancels. \\
Matched transitions & Clarification after transition results were known & Origins are month $t$ for transitions from $t$ to $t+1$; unclassified occupations remain separate. & The entry comparison is reported with its support and selection limits. \\
\bottomrule
\end{tabular}
\end{table}

\clearpage
\section*{Figure legends}
\begingroup\raggedright
\noindent Figure~\ref{fig:stockflow}. \csname legendstockflow\endcsname\par\medskip
\noindent Figure~\ref{fig:headline}. \csname legendheadline\endcsname\par\medskip
\noindent Figure~\ref{fig:pretrends}. \csname legendpretrends\endcsname\par\medskip
\noindent Figure~\ref{fig:agegradient}. \csname legendagegradient\endcsname\par
\endgroup

\clearpage
\setcounter{section}{0}
\renewcommand{\thesection}{S\arabic{section}}
\renewcommand{\theHsection}{SI.\arabic{section}}
\setcounter{table}{0}
\renewcommand{\thetable}{S\arabic{table}}
\renewcommand{\theHtable}{S\arabic{table}}

\begin{center}
{\Large\bfseries Supporting Information for ``Young Workers in Jobs Exposed to Generative Artificial Intelligence: Employment Flows and Historical Comparisons''}
\end{center}

This appendix supplies construction details, additional estimates and sensitivity analyses. It contains no author-identifying information. The main article gives the final design and the main limitations; this document provides the calculations needed to interpret each table.

\section{CPS construction and inference}

The primary CPS occupation measure uses the official detailed 2018 SOC universe and the May 2022 OEWS staffing weights. A scored member without an observed employment weight is dropped from its Census-family mean. If no scored weighted member remains, the Census occupation remains unavailable. The former global-median imputation of \oewsMedianWeight{} is retained as a labelled sensitivity with the complete 499-replication harmonised-occupation-group bootstrap. The non-imputing primary rule does not guarantee a direction for the estimate.

For all occupation-labelled CPS contrasts, uncertainty uses 499 mean-one Poisson replications by final harmonised occupation group. Each group is a connected component of the Census crosswalk, so records linked across the 2010 and 2018 classifications remain together. The former household bootstrap, state-cluster uncertainty and population-composition checks are separate diagnostics. Confidence intervals and hypothesis tests are pointwise, without multiplicity adjustment, and do not claim simultaneous coverage across measures, windows or specifications.

\begin{table}[!htbp]\centering
\caption{Support under the primary and former median-imputation OEWS member-weight rules$^{\dagger}$}
\label{tab:missing_weight_support}\small
\resizebox{\textwidth}{!}{\begin{tabular}{@{}lrrrrrrrr@{}}\toprule
Measure & Affected families & Affected emp. (\%) & \multicolumn{3}{c}{Primary drop rule} & \multicolumn{3}{c}{Former median sensitivity} \\ \cmidrule(lr){4-6}\cmidrule(lr){7-9} & & & Rows & Source codes & Coverage (\%) & Rows & Source codes & Coverage (\%) \\ \midrule
GPT-4 & 18 & 4.20 & 2,466,302 & 603 & 96.2 & 2,513,006 & 617 & 98.2 \\
observed use & 1 & 0.08 & 2,442,083 & 596 & 95.1 & 2,442,083 & 596 & 95.1 \\
capability & 18 & 4.20 & 2,460,062 & 601 & 95.9 & 2,506,766 & 615 & 98.0 \\
\bottomrule\end{tabular}}
\begin{minipage}{0.95\textwidth}\footnotesize\smallskip
\textit{Notes:} Affected Census occupation codes are GPT-4: 510, 520, 530, 810, 2001, 2004, 2330, 2545, 3010, 3300, 3601, 3602, 4540, 6115, 6765, 7720, 7750, 9005; observed use: 3010; capability: 510, 520, 530, 810, 2001, 2004, 2330, 2545, 3010, 3300, 3601, 3602, 4540, 6115, 6765, 7720, 7750, 9005. The six mixed observed/unavailable-weight families are 2330, 2545, 3010, 7720, 7750, 9005. Employment shares and coverage use employed CPS respondents aged 16--64 in 2022; rows are person-month observations in the 2018--2026 analysis frame. Source-code counts are distinct source occupation-code values across the 2010/2018 classifications, not counts of harmonised 2018 occupations; exposure values for earlier records use the official weighted bridge while the source code is retained. The former global median is 42,865; both reported sensitivity point estimates and standard errors use 499 final-harmonised-occupation-group bootstrap replications.
\end{minipage}\end{table}

\section{Institutional and benchmark results}

Union membership is measured in pooled January 2018--October 2022 CPS outgoing-rotation observations and weighted by the CPS outgoing-rotation survey weight among wage-and-salary employees. Cells require at least one hundred respondents. The above- and below-median groups are applied to the balanced QWI headline sample. The split remains descriptive and non-causal.

\begingroup
\let\bjirOriginalLabel\label
\renewcommand{\label}[1]{\bjirOriginalLabel{si:#1}}
\begin{table}[!htbp]\centering
\caption{Young-worker employment contrast by pre-outcome union density$^{\dagger}$}
\label{tab:union_density}\small
\resizebox{\textwidth}{!}{\begin{tabular}{@{}lccc@{}}\toprule
 & Estimate (state-clustered s.e.) & Wild \(p\) & States \\ \midrule
Above-median union membership & $-0.227$ $(0.020)$ & 0.001 & 49 \\
Below-median union membership & $-0.163$ $(0.010)$ & 0.001 & 49 \\
Above minus below & $-0.064$ $(0.016)$ & 0.001 & 49 \\
\bottomrule\end{tabular}}
\begin{minipage}{0.95\textwidth}\footnotesize\smallskip
\textit{Notes:} Union membership is weighted by the CPS outgoing-rotation survey weight among employed wage-and-salary workers, pooled January 2018--October 2022 at state by two-digit NAICS. Cells require at least 100 respondents. The median (0.049) was fixed before linking QWI outcomes. Usable cells cover 97.4 percent of above-median and 96.9 percent of below-median linked base employment. Estimates use the paper's balanced QWI headline sample and are descriptive, not causal. Wild \(p\)-values use 999 Rademacher multiplier draws on state influence functions. $^{\dagger}$Exploratory; specified after an interim diagnostic. The outgoing-rotation weight was subsequently corrected.
\end{minipage}\end{table}

\endgroup

Average monthly earnings for new hires who reach full-quarter stable employment are a selected outcome. Cell means require positive published counts and earnings, and the same age-by-exposure contrast is applied. These values omit failed or short matches and do not measure bargaining or first labour-market entry.

\begingroup
\let\bjirOriginalLabel\label
\renewcommand{\label}[1]{\bjirOriginalLabel{si:#1}}
\begin{table}[!htbp]\centering
\caption{Earnings of new hires reaching full-quarter stable employment, 2023--2025$^{\dagger}$}
\label{tab:entry_terms}\small
\resizebox{\textwidth}{!}{\begin{tabular}{@{}lc@{}}\toprule
 &  \\ \midrule
\multicolumn{2}{@{}l}{\textit{Average monthly earnings of new hires reaching stable employment}} \\
 & Change in log earnings (bootstrap s.e.) \\
\quad Young, exposed & $0.011$ $(0.014)$ \\
\quad Young, other & $0.048$ $(0.009)$ \\
\quad Older, exposed & $0.024$ $(0.008)$ \\
\quad Older, other & $0.035$ $(0.014)$ \\
\quad Young: exposed minus other & $-0.037$ $(0.017)$ \\
\quad Older: exposed minus other & $-0.010$ $(0.016)$ \\
\quad Triple difference & $-0.027$ $(0.017)$ \\
\quad Industry--state pairs & 5,509 \\
\quad States; industries & 49; 210 \\
\bottomrule\end{tabular}}
\begin{minipage}{0.95\textwidth}\footnotesize\smallskip
\textit{Notes:} QWI EarnHirNS, the average monthly earnings of new hires who reach full-quarter stable employment, aggregated with HirNS counts as weights. Each row is the change from the 2021Q4--2022Q3 base to the 2023Q1--2025Q3 outcome window. Ages are 22--24 and 35--44; exposed industries are the fixed top fifth by May 2022 GPT-4 exposure. The triple difference is (young exposed minus young other) minus (older exposed minus older other). Industry-bootstrap standard errors, 499 draws. $^{\dagger}$Exploratory; specified after the initial estimates were examined.
\end{minipage}\end{table}

\endgroup

The New York Fed recent-graduate series is a benchmark outside the exposure design. It has no artificial-intelligence-exposure split and cannot validate the exposure gradient or its timing.

\begin{table}[!htbp]\centering
\caption{Independent recent-graduate labour-market benchmark}
\label{tab:strengthenBThree}\small
\begin{tabular}{@{}p{0.30\linewidth}p{0.62\linewidth}@{}}\toprule
Series & 2026Q2 NY Fed published snapshot (accessed 19 October 2026) \\ \midrule
NY Fed recent college graduates & 2026Q2: unemployment about 5.6\%; underemployment about 42\% \\
\bottomrule\end{tabular}
\par\smallskip\begin{minipage}{0.95\linewidth}\footnotesize\textit{Notes:} Federal Reserve Bank of New York quarterly highlights \citep{nyfed2026college}. Recent graduates are US bachelor's-degree holders aged 22--27 who are not enrolled in school. Underemployment means work in an occupation that typically does not require a bachelor's degree; it is not unemployment. The homepage labels both values as approximate. The 2019 same-quarter comparison is not reported because the historical workbook was unavailable. The series has no artificial-intelligence-exposure split.\end{minipage}
\end{table}

\section{Use type, pre-trends and monetary sensitivity}

The automating-minus-complementary comparison applies the Anthropic use classification to the CPS. The classification comes from Claude conversations reported in March 2025, after much of the employment outcome period. Alignment with this later use composition cannot establish adoption or treatment status fixed before the employment changes. Source O*NET-SOC 2010 codes remain eight-digit through the official 2010-to-2019 bridge and are reduced to six-digit 2018 SOC only at the documented aggregation step. A Census occupation enters the no-recorded-use group only when the fixed member table finds no mapped recorded-use member. The primary exposure and the separate type-of-use analysis use different crosswalk rules. The rate-sensitivity exercise uses the first 2018 target in the official Census change-block order for records before January 2020 and the observed code thereafter.

\begin{table}[!htbp]\centering
\caption{Competing explanations and the type of AI use$^{\dagger}$}
\label{tab:compete}\small
\resizebox{\textwidth}{!}{\begin{tabular}{@{}lcc@{}}\toprule
 & &  \\ \midrule
\multicolumn{3}{@{}l}{\textit{A. Interest-rate sensitivity}} \\
 & CPS & QWI \\
\quad Exposure contrast, most rate-sensitive fifth excluded & $-0.120$ $(0.065)$ & $-0.195$ $(0.044)$ \\
\quad Rate-sensitivity sort instead of exposure & $0.097$ $(0.053)$ & $-0.029$ $(0.042)$ \\
\quad First stage: high-sensitivity fifth minus rest, total QWI employment, 2022--2024 & & $-0.033$ $(0.019)$ \\
\multicolumn{3}{@{}l}{\textit{B. Automating minus complementary use, both against occupations with no recorded use (CPS)}} \\
 & All occupations & Without computer and mathematical \\
\quad Jul 2025--Sep 2026 & $-0.045$ $(0.077)$ & $-0.027$ $(0.081)$ \\
\quad Pre-period placebo: mean of 2018--2021 & $-0.071$ $(0.067)$ & $-0.042$ $(0.065)$ \\
\quad 2019 & $-0.149$ $(0.114)$ & $-0.107$ $(0.115)$ \\
\quad 2023 & $0.034$ $(0.059)$ & $0.071$ $(0.056)$ \\
\quad 2024 & $0.051$ $(0.060)$ & $0.078$ $(0.060)$ \\
\quad Automating use alone, Jul 2025--Sep 2026 & $-0.051$ $(0.078)$ & $-0.037$ $(0.087)$ \\
\quad Complementary use alone, Jul 2025--Sep 2026 & $-0.006$ $(0.072)$ & $-0.010$ $(0.074)$ \\
\multicolumn{3}{@{}l}{\textit{C. Rambachan--Roth breakdown value (CPS, GPT-4; estimands differ by binning)}} \\
\quad Annual bins: Jul 2025--Sep 2026; Jan--Oct 2022 base & 0.07 & \\
\quad Six-month bins: Jul 2025--Jun 2026; Jan--Jun 2022 base & 0.00 & \\
\quad Stable-code occupations, pre-period 2018--2019 only & 0.00 & \\
\multicolumn{3}{@{}l}{\textit{D. Rambachan--Roth, automating minus complementary contrast (CPS)}} \\
\quad Target, Jul 2025--Sep 2026 & -0.045 & \\
\quad Original 95\% interval & [-0.198, 0.108] & \\
\quad Relative-magnitude 95\% interval, $\overline{M}=1$ & [-1.563, 1.488] & \\
\quad Breakdown value & 0.00 & \\
\bottomrule\end{tabular}}
\begin{minipage}{0.95\textwidth}\scriptsize\smallskip
\textit{Notes:} Panel A: interest-rate sensitivity is minus each 4-digit industry's eight-quarter QWI employment response to the orthogonalised monetary-policy surprises of \citet{bauer2023reassessment}, 2000--2019 (300 industries); occupations take the May 2022 OEWS staffing-weighted mean of their industries' sensitivity (100 percent of 2022 CPS employment mapped). All 867 official detailed 2018 SOC occupations receive one final Census owner; separately classified occupations are excluded from residual X categories before May 2022 OEWS weighting. Panels A and B use the 25 October 2026 non-imputing rule, whose estimate direction is not guaranteed: an unweighted scored member is dropped from its Census-family mean, and a family with no scored weighted member remains unavailable. Of 830 positive source-weight keys, 818 match the 867 official members and 12 do not; 49 official members lack weights. The main CPS table also reports the fully bootstrapped former-median sensitivity. The occupation estimates of \citet{zens2020heterogeneous} are published only as figures. CPS cells use the main article's late-window triple-difference specification; QWI cells use its change in log employment to 2025Q4. The first stage is the 2023Q1--2024Q4 change from a 2021Q4--2022Q3 base in total QWI employment in the most rate-sensitive fifth relative to the rest, across 300 industries; its intended sign is negative. Panels A and B: before January 2020, a 2010 occupation with multiple Census crosswalk targets is assigned to the first 2018 target in the official change-block order; observations from January 2020 already use 2018 occupation codes. All 974 source O*NET-SOC 2010 codes remain eight-digit through the official 2010-to-2019 O*NET bridge and are reduced to six-digit 2018 SOC only at the documented aggregation step. A Census occupation is classified as having no recorded use only after the fixed member table finds no mapped recorded-use member. The panel uses the top employment-weighted fifth of occupations by the share of their AI conversations classified as automating (directive, feedback loop) or complementary (task iteration, validation, learning) in the Anthropic Economic Index (March 2025), each against the 104 occupations with no recorded conversations; both arms use the same final-harmonised-occupation-group draws (499), while the former household draws are retained as a sensitivity. Computer and mathematical occupations: 2018 Census codes 1000--1240. Panel C reports relative-magnitude sensitivity to departures from parallel trends. Its annual and six-month rows are not the same estimand: the six-month row uses a Jan--Jun 2022 base and equal weights of 0.5 on 2025H2 and 0.5 on 2026H1; Oct 2025 (not collected) is absent. Panel D applies the same calculation to the automating-minus-complementary annual-bin contrast and reports the unadjusted and $\overline{M}=1$ intervals. $^{\dagger}$Exploratory; specified after the main estimates were known. The SOC-pattern, Census change-block, original-order first-target and O*NET-vintage corrections were made after all results were known.
\end{minipage}\end{table}

The annual and six-month Rambachan--Roth \citep{rambachan2023credible} rows are distinct estimands. They do not isolate bin width because the six-month row also changes the base, endpoint and target-half-year weights. The monetary exercise is a construct-validity check; its weak first stage and wrong-signed sensitivity ranking prevent it from ruling out monetary transmission.

\section{Payroll and transition benchmarks}

The direct CPS analogue compares the top two GPT-4 quintiles with the bottom three. The later young-worker CPS result is compatible with the payroll magnitude, but source definitions and timing differ. Matched transitions are assigned by their origin month. Periods refer to the origin month $t$ of a transition from $t$ to $t+1$. October 2025 was not collected, and September 2026 has no next-month record. Entry from non-employment and movement from scored occupations outside the top fifth are reported separately from transitions involving occupations without an exposure score.

\begin{table}[!htbp]\centering
\caption{Benchmarks: the payroll statistic in the CPS, and young workers' flows$^{\dagger}$}
\label{tab:bench}\small
\resizebox{\textwidth}{!}{\begin{tabular}{@{}lcc@{}}\toprule
 & &  \\ \midrule
\multicolumn{3}{@{}l}{\textit{A. CPS analogue of the payroll statistic: log employment, top two GPT-4 quintiles minus bottom three, Oct--Dec 2022 to Apr--Jun 2026}} \\
 & Estimate & 95\% interval \\
\quad 22--25 & $-0.144$ & $[-0.257, -0.042]$ \\
\quad 30--49 & $0.054$ & $[-0.005, 0.104]$ \\
\quad Payroll data \citep{brynjolfsson2026canaries}, 22--25: 19 percent below kept-pace level & $-0.211$ & \\
\multicolumn{3}{@{}l}{\textit{B. Matched month-to-month CPS records, 22--25 in scored top-quintile occupations: monthly transition rates}} \\
 & Base-origin months & Change in late-origin period, relative to 30--49 \\
\quad Leave top fifth for a scored lower-exposure occupation & 0.0554 & $0.0089$ $(0.0063)$ \\
\quad Leave top fifth for an occupation without an exposure score & 0.0043 & $-0.0011$ $(0.0015)$ \\
\quad To unemployment & 0.0110 & $-0.0004$ $(0.0020)$ \\
\quad Out of the labour force & 0.0385 & $-0.0051$ $(0.0060)$ \\
\quad Into top-quintile jobs from non-employment & 0.0265 & $0.0003$ $(0.0040)$ \\
\quad Into top-quintile jobs from a scored lower-exposure occupation & 0.0159 & $-0.0009$ $(0.0020)$ \\
\quad Into top-quintile jobs from an occupation without an exposure score & 0.0217 & $0.0001$ $(0.0105)$ \\
\multicolumn{3}{@{}l}{\textit{C. Entry from non-employment, relative to its young-worker base rate}} \\
 & Estimate & Normal 95\% interval \\
\quad Percentage change & 1\% & [-28\%, 31\%] \\
\bottomrule\end{tabular}}
\begin{minipage}{0.95\textwidth}\footnotesize\smallskip
\textit{Notes:} Panel A: CPS log employment ratio of the two most exposed GPT-4 quintiles to the bottom three, change between three-month averages; final harmonised occupation-group interval (499). Panel B: each matched record is assigned by its origin month $t$ and measures the transition from $t$ to $t+1$. Observed origin supports are January--October 2022 (base), November 2022--June 2025 (mid) and July--August 2025 and November 2025--August 2026 (late); the corresponding destination supports are February--November 2022, December 2022--July 2025 and August--September 2025 and December 2025--September 2026. October 2025 was not collected, so neither the September--October nor October--November transition is observed; September 2026 has no next-month record and therefore is not an origin observation. Top and lower-exposure occupations must have a nonmissing exposure score; employed occupations without a score are reported separately rather than included in the lower-exposure comparison. A move between two top-quintile occupation codes remains in the top-quintile state. Among employed young-worker origins outside the top fifth, occupations without an exposure score carry 5.2 percent of base-period weight and 4.9 percent of late-period weight. Rates are for 22- to 25-year-olds employed in (or entering) scored top-quintile occupations; the change is (late minus base, young) minus (late minus base, 30--49). The 499-replication primary bootstrap uses the final harmonised Census crosswalk component at the origin for exit statistics and at the destination for entry statistics. A transition endpoint without an occupation receives no artificial occupation label and is held fixed in this bootstrap; the former household bootstrap is retained as a sensitivity for dependence in the full risk set. Panel C divides the entry-from-non-employment change and its normal interval by the corresponding young-worker base rate. The extract has no enrolment variable. $^{\dagger}$Specified after the main estimates were known.
\end{minipage}\end{table}

\section{Additional specifications}

The cell-panel regressions and the window statistics are different estimands. The former equally weight observed positive cells after fixed-effect projection; the latter aggregate employment within age--exposure groups before taking logs. The following tables give the associated flow and placebo sensitivities.

\begin{table}[!htbp]\centering
\caption{Alternative comparisons for the QWI flow reconciliation$^{\dagger}$}
\label{tab:recon_appendix}\small
\resizebox{\textwidth}{!}{\begin{tabular}{@{}lcc@{}}\toprule
 & Estimate (s.e.) &  \\ \midrule
\multicolumn{3}{@{}l}{\textit{Alternative comparisons}} \\
 & Quarter-end-retained hires & Continuing-job separations \\
\quad 22--24 only, no older group: employment $-$0.211 & $-0.200$ $(0.058)$ & $-0.068$ $(0.058)$ \\
\quad 22--24 vs 25--54: employment $-$0.156 & $-0.135$ $(0.033)$ & $-0.064$ $(0.035)$ \\
\quad Cumulated from 2022Q4: employment $-$0.168 & $-0.171$ $(0.033)$ & $-0.074$ $(0.042)$ \\
\quad Rates not de-seasonalised: employment $-$0.146 & $-0.119$ $(0.034)$ & $-0.079$ $(0.039)$ \\
\bottomrule\end{tabular}}
\begin{minipage}{0.95\textwidth}\footnotesize\smallskip
\textit{Notes:} Each row changes one element of the main article's 22--24 versus 35--44 flow comparison. Entries are contributions to the employment triple difference under the quarter-end-retained-hire and continuing-job-separation definitions. $^{\dagger}$Specified after the main estimates were known.
\end{minipage}\end{table}

\subsection{Flow rates and proportional counterfactual}

Level-rate contributions add to the employment accounting identity. They are not proportional hiring-rate effects. In the recent sample, all four cell log rates fall, while the older-proportional comparison is positive. Quarter-end retention records whether a hire is present at the quarter boundary; it does not hold tenure fixed and cannot identify match duration.

\begingroup
\let\bjirOriginalLabel\label
\renewcommand{\label}[1]{\bjirOriginalLabel{si:#1}}

\endgroup

\subsection{Placebo-window sensitivity}

The exploratory placebo windows and the recent episode differ in their one-quarter gap, accumulated-flow quarters and seasonal-reference offsets. A later sensitivity matches those offsets while retaining the original candidate-start range, common-state sample and recession-exclusion rule. The recent employment contrast and within-quarter hire term remain outside the matched central ranges, while both flow-gap measures remain inside.

\begin{table}[!htbp]\centering
\caption{Later placebo sensitivity matching the recent episode's offsets}
\label{tab:placebo_matched}\small
\resizebox{\textwidth}{!}{\begin{tabular}{@{}lccccc@{}}\toprule
 & Original 2.5th & Original 97.5th & Matched 2.5th & Matched 97.5th & Recent \\ \midrule
\quad Employment triple difference & $-$0.084 & 0.083 & $-$0.088 & 0.105 & $-$0.178 \\
\quad Within-quarter hire contribution & $-$0.094 & 0.029 & $-$0.091 & 0.039 & 0.102 \\
\quad All separations minus all hires & $-$0.304 & 0.269 & $-$0.336 & 0.252 & $-$0.144 \\
\quad Continuing-job separations minus quarter-end-retained hires & $-$0.245 & 0.125 & $-$0.259 & 0.127 & 0.059 \\
\bottomrule\end{tabular}}
\begin{minipage}{0.95\textwidth}\footnotesize\smallskip
\textit{Notes:} The matched sensitivity copies every recent-episode offset: seasonal reference $-24$ to $-13$, base $-5$ to $-2$, a one-quarter gap, eleven flow quarters $0$ to $10$, and stock endpoint 11. It uses the same candidate-start range, common states and NBER-recession exclusion rule as the internally specified exercise. The final supports overlap for 29 of 30 windows: the matched rule adds 2004Q4 and omits 2010Q3. This diagnostic was added after the main estimates were known and is not described as pre-specified.
\end{minipage}\end{table}

\clearpage
\section{Analysis plans, deviations and corrections}

This section collects the analysis history in one place. The timing statements describe the record and do not change the interpretation of any estimate.

\begin{table}[!htbp]
\centering\scriptsize
\caption{Analysis plans, deviations and corrections}
\label{tab:si_analysis_history}
\resizebox{\textwidth}{!}{\begin{tabular}{@{}P{.2\textwidth}P{.2\textwidth}P{.3\textwidth}P{.2\textwidth}@{}}
\toprule
Analysis & Timing relative to the main results & Deviation or correction & Effect on reported results \\
\midrule
CPS weights and occupation linkage & The missing-member treatment was found to be unspecified after the earlier estimates were known; the correction rule was fixed before recomputation. The detailed-SOC correction was made after all results were known. & A scored member without an observed employment weight is dropped. The former global-median imputation remains a labelled sensitivity. Exact detailed 2018 SOC codes from the fixed detailed occupation table are used. & CPS estimates are reported under both weight rules. \\
QWI capability linkage & Correction after the main results & The canonical detailed-SOC source replaces the former harmonised panel; missing scores remain unavailable. This correction does not change the QWI exact-detailed-SOC score policy. & The secondary capability estimate changes; primary results do not. \\
Flow, placebo and institutional checks & After the initial or split estimates were known & The flow and placebo rules were specified after the initial results. The union-density split detects descriptive heterogeneity; the paired industry bootstrap, specified after the split estimates were known, supplies the right-to-work contrast. & Results remain exploratory, descriptive and non-causal. \\
QWI flow availability & Availability correction recorded on 19 November 2026, after the results were known & The original plan specified 2025Q4 for $SepBeg$. That quarter was unpublished, so the implementation uses a common flow cutoff of 2025Q3; employment ends in 2025Q4. & This corrects the documented endpoint and leaves the estimates unchanged. \\
Composition and transitions & The plan specified five-year age bands crossed with sex; transition timing was clarified after results were known & The implemented check uses single-year ages crossed with sex. Origins are month $t$ for a transition from $t$ to $t+1$ and unclassified occupations remain separate. & These diagnostics retain their support and selection limits. \\
\bottomrule
\end{tabular}}
\end{table}

The original CPS plan, dated before estimation, specified the November 2022--June 2025 and July 2025--September 2026 point contrasts and the unemployment-share contrast. The first two subdivisions of the first post period, bootstrap inference and the component panels were added after the main estimates were known. The corrected detailed-SOC universe contains \oewsOfficialMembers{} official members, with \oewsMissingWeights{} official members lacking a recorded weight; the sensitivity uses the former global median. The authors' dated records place the flow and placebo specifications before those analyses; that order is not independently authenticated, and the timing statements in the table notes carry the same qualification. The QWI capability linkage and the missing-weight, residual-membership, Census change-block and O*NET-vintage corrections were kept whichever way they moved the estimates. The supporting methods assign all \oewsOfficialMembers{} official detailed 2018 SOC occupations to mutually exclusive Census occupations.

The origin month $t$ convention was clarified after the transition results were known. The complete QWI flow samples additionally require every required sequence to be published and strictly positive in every included cell; legitimate published zero flows are therefore excluded by design. The reported flow estimates are conditional on this strictly positive sample.

\bibliographystyle{plainnat}
\bibliography{references/bibliography}

\end{document}